\documentclass[onecolumn,10pt]{article}
\usepackage{jmlr2e,booktabs,enumitem,epstopdf,threeparttable,multirow,algorithmic,nameref,url}
\usepackage[usenames]{color}
\usepackage[normalem]{ulem}
\usepackage[cmex10]{amsmath}
\usepackage[tight,footnotesize]{subfigure}
\usepackage[ruled,vlined]{algorithm2e}
\usepackage[titletoc,title]{appendix}

\newcommand{\bzero}{\mathbf{0}}

\newcommand{\ba}{\mathbf{a}}
\newcommand{\bA}{\mathbf{A}}

\newcommand{\bC}{\mathbf{C}}

\newcommand{\E}{\mathbb{E}}

\newcommand{\cF}{\mathcal{F}}

\newcommand{\bI}{\mathbf{I}}

\newcommand{\cN}{\mathcal{N}}

\newcommand{\bbP}{\mathbb{P}}

\newcommand{\cP}{\mathcal{P}}
\newcommand{\bq}{\mathbf{q}}
\newcommand{\bQ}{\mathbf{Q}}

\newcommand{\R}{\mathbb{R}}

\newcommand{\bbS}{\mathbb{S}}

\newcommand{\bu}{\mathbf{u}}

\newcommand{\cU}{\mathcal{U}}
\newcommand{\bv}{\mathbf{v}}

\newcommand{\bV}{\mathbf{V}}
\newcommand{\cV}{\mathcal{V}}
\newcommand{\bcV}{\boldsymbol{\cV}}

\newcommand{\bx}{\mathbf{x}}

\newcommand{\Z}{\mathbb{Z}}

\newcommand{\bSigma}{\boldsymbol{\Sigma}}

\newcommand{\bxi}{\boldsymbol{\xi}}

\DeclareMathOperator*{\argmax}{arg\,max}

\newcommand{\tT}{\mathrm{T}}

\newcommand{\bPr}{\bbP}

\newcommand{\DK}{\mbox{D-Krasulina}}
\newcommand{\DMK}{\mbox{DM-Krasulina}}
\newcommand{\qed}{\hfill$\blacksquare$}
\ShortHeadings{Distributed Algorithms for High-rate Streaming PCA}{Raja and Bajwa}

\firstpageno{1}
\begin{document}

\title{Distributed Stochastic Algorithms for High-rate Streaming Principal Component Analysis}	

\author{Haroon~Raja%
    \thanks{Haroon Raja is a research fellow in the Department of Electrical Engineering and Computer Science at the University of Michigan, Ann Arbor, MI (Email: {\tt hraja@umich.edu}).}%
    ~~and~Waheed~U.~Bajwa%
    \thanks{(Corresponding Author) Waheed U. Bajwa is an associate professor in the Department of Electrical and Computer Engineering at Rutgers University--New Brunswick, NJ (Email: {\tt waheed.bajwa@rutgers.edu}).}~%
\thanks{The work reported in this paper has been supported in part by the National Science Foundation under awards CCF-1453073, CCF-1907658, and OAC-1940074, by the Army Research Office under award W911NF-17-1-0546, and by the DARPA Lagrange Program under ONR/NIWC contract N660011824020.}
}

\maketitle

\begin{abstract}
This paper considers the problem of estimating the principal eigenvector of a covariance matrix from independent and identically distributed data samples in streaming settings. The streaming rate of data in many contemporary applications can be high enough that a single processor cannot finish an iteration of existing methods for eigenvector estimation before a new sample arrives. This paper formulates and analyzes a distributed variant of the classical Krasulina's method (\DK) that can keep up with the high streaming rate of data by distributing the computational load across multiple processing nodes. The analysis shows that---under appropriate conditions---\DK~converges to the principal eigenvector in an order-wise optimal manner; i.e., after receiving $M$ samples across all nodes, its estimation error can be $O(1/M)$. In order to reduce the network communication overhead, the paper also develops and analyzes a mini-batch extension of \DK, which is termed \DMK. The analysis of \DMK~shows that it can also achieve order-optimal estimation error rates under appropriate conditions, even when some samples have to be discarded within the network due to communication latency. Finally, experiments are performed over synthetic and real-world data to validate the convergence behaviors of \DK~and \DMK~in high-rate streaming settings.
\end{abstract}


\section{Introduction}
\label{sec:Introduction}
Dimensionality reduction and feature learning methods such as \emph{principal component analysis} (PCA), sparse PCA, independent component analysis, and autoencoder form an important component of any machine learning pipeline. For data lying in a $d$-dimensional space, such methods try to find the \mbox{$k \ll d$} variables/features that are most relevant for solving an application-specific task (e.g., classification, regression, estimation, data compression, etc.). The focus of this work is on PCA, where the objective is to compute $k$-features that capture most of the variance in data. The proliferation of \emph{big data} (both in terms of dimensionality and number of samples) has resulted in an increased interest in developing new algorithms for PCA due to the fact that classical numerical solutions (e.g., power iteration and Lanczos method~\citep{golub2012matrix}) for computing eigenvectors of symmetric matrices do not scale well with high dimensionality and large sample sizes. The main interest in this regard has been on developing algorithms that are cheap in terms of both memory and computational requirements as a function of dimensionality and number of data samples.

In addition to high dimensionality and large number of samples, another defining characteristic of modern data is their streaming nature in many applications; examples of such applications include the internet-of-things, high-frequency trading, meteorology, video surveillance, autonomous vehicles, social media analytics, etc. Several stochastic methods have been developed in the literature to solve the PCA problem in streaming settings~\citep{krasulina1969method,oja,sanger1989optimal,Warmuth2007,zhang2016global}. These methods operate under the \emph{implicit} assumption that the data arrival rate is slow enough so that each sample can be processed before the arrival of the next one. But this may not be true for many modern applications involving high-rate streaming data. To overcome this obstacle corresponding to high-rate streaming data, this paper proposes and analyzes distributed and distributed, mini-batch variants of the classical Krasulina's method~\citep{krasulina1969method}. Before providing details of the proposed methods and their relationship to prior work, we provide a brief overview of the streaming PCA problem.

\subsection{Principal Component Analysis (PCA) from Streaming Data}\label{subsec:ReviewPCA}
For data lying in $\mathbb{R}^d$, PCA learns a $k$-dimensional subspace with maximum data variance. Let $\bx\in\R^d$ be a random vector that is drawn from some unknown distribution $\mathcal{P}_{x}$ with zero mean and $\bSigma$ covariance matrix. For the constraint set $\bcV:=\{\bV\in\mathbb{R}^{d\times k}:\;\bV^{\tT}\bV=\bI\}$, we can pose PCA as the following constrained optimization problem:
\begin{align}\label{eqn:PCA}
\bQ^* := \argmax_{\bV \in \bcV} \E_{\cP_{x}}\Big\{\textsf{Tr}(\bV^{\tT}\bx\bx^\tT \bV)\Big\},
\end{align}
where $\textsf{Tr}(.)$ denotes the trace operator. The solution for the \emph{statistical risk maximization} problem~\eqref{eqn:PCA} is the matrix $\bQ^*$ with top $k$ eigenvectors of $\bSigma$. In practice, however, \eqref{eqn:PCA} cannot be solved in its current form since $\mathcal{P}_x$ is unknown. But if we have $T$ data samples, $\{\bx_t\}_{t=1}^T$, drawn independently from $\cP_x$, then we can accumulate these data samples to calculate the sample covariance matrix as:
\begin{align}
    \bar{\bA}_T :=\frac{1}{T}\sum_{t=1}^{T}{\bA_t},
\end{align}
where $\bA_t:=\bx_t \bx_{t}^{\tT}$. Instead of solving~\eqref{eqn:PCA}, we can now solve an \emph{empirical risk maximization} problem
\begin{align}\label{eqn:ERM}
\bQ := \argmax_{\bV \in \bcV} \textsf{Tr}(\bV^{\tT} \bar{\bA}_T \bV) = \argmax_{\bV \in \bcV}\frac{1}{T}\sum_{t=1}^{T}\textsf{Tr}(\bV^{\tT}\bA_t \bV).
\end{align}
In principle, we can solve \eqref{eqn:ERM} by computing the \emph{singular value decomposition} (SVD) of sample covariance $\bar{\bA}_T$. But this is a computationally intensive task that requires $O(d^3)$ multiplications and that has a memory overhead of $O(d^2)$. In contrast, the goal in high-dimensional PCA problems is often to have $O(d^2 k)$ computational complexity and $O(dk)$ memory complexity~\citep{pmlr-v51-li16b}.

More efficient (and hence popular) approaches for PCA use methods such as the power/orthogonal iteration and Lanczos method~\citep[Chapter~8]{golub2012matrix}. Although these methods improve overall computational complexity of PCA to $O(d^2 k)$, they still have memory requirements on the order of $O(d^2)$. In addition, these are \emph{batch} methods that require computing the sample covariance matrix $\bar{\bA}_T$, which results in $O(d^2 T)$ multiplication operations. Further, in streaming settings where the goal is real-time decision making from data, it is infeasible to compute $\bar{\bA}_T$. Because of these reasons, stochastic approximation methods such as Krasulina's method~\citep{krasulina1969method} and Oja's rule~\citep{oja} are often favored for the PCA problem. Both these are simple and extremely efficient algorithms, achieving $O(d)$ computational and memory complexity per iteration, for computing the principal eigenvector (i.e., $k=1$) of a covariance matrix in streaming settings. Recent years in particular have seen an increased popularity of these algorithms and we will discuss these recent advances in Section~\ref{subsec:RelatedWork}.

Both Oja's rule and Krasulina's method share many similarities. In this paper, we focus on Krasulina's method with the understanding that our findings can be mapped to Oja's rule through some tedious but straightforward calculations. Using $t$ for algorithmic iteration, Krasulina's method estimates the top eigenvector by processing one data sample in each iteration as follows:\footnote{In contrast, the iterate of Oja's rule is given by $\bv_t = \bv_{t-1} + \gamma_t \left(\bx_t \bx_t^\tT \bv_{t-1} - \bv_{t-1}^{\tT}\bx_t \bx_t^\tT \bv_{t-1} \bv_{t-1}\right).$}
\begin{align}\label{eqn:Krasulina}
\bv_t = \bv_{t-1} + \gamma_t \Bigg(\bx_t\bx_t^\tT \bv_{t-1} - \frac{\bv_{t-1}^{\tT}\bx_t\bx_t^\tT \bv_{t-1} \bv_{t-1}}{\|\bv_{t-1}\|_2^2}\Bigg),
\end{align}
where $\gamma_t$ denotes the step size. Going forward, we will be using $\bA_t$ in place of $\bx_t \bx_t^\tT$ in expressions such as~\eqref{eqn:Krasulina} for notational compactness. In practice, however, one should neither explicitly store $\bA_t$ nor explicitly use it for calculation purposes.

Note that one can interpret Krasulina's method as a solution to an optimization problem. Using Courant--Fischer Minimax Theorem~\cite[Theorem~8.1.2]{golub2012matrix}, the top eigenvector computation (i.e., \emph{1-PCA}, which is the $k=1$ version of~\eqref{eqn:PCA}) can be posed as the following optimization problem:
\begin{align}\label{eqn:RayleighQuotient}
\bq_1 := \arg\min_{\bv\in\R^d}f(\bv)=\arg\min_{\bv\in\R^d}\frac{-\bv^{\tT} \bA_t \bv}{\|\bv\|_2^2}.
\end{align}
In addition, the gradient of the function $f(\bv)$ defined in~\eqref{eqn:RayleighQuotient} is:
\begin{align}\label{eqn:Gradient}
\nabla f(\bv)= \frac{1}{\|\bv\|_2^2}\Bigg(- \bA_t \bv + \frac{(\bv^{\tT}\bA_t \bv) \bv}{\|\bv\|_2^2}\Bigg).
\end{align}
Looking at~\eqref{eqn:Krasulina}--\eqref{eqn:Gradient}, we see that~\eqref{eqn:Krasulina} is very similar to applying \emph{stochastic gradient descent} (SGD) to the nonconvex problem~\eqref{eqn:RayleighQuotient}, with the only difference being the scaling factor of $1/\|\bv\|_2^2$. Nonetheless, since \eqref{eqn:RayleighQuotient} is a nonconvex problem and we are interested in global convergence behavior of Krasulina's method, existing tools for analysis of the standard SGD problem~\citep{Bottou.ConfCOMPSTAT10,recht:NIPS11,dekel:JMLR12,reddi2016bfast,reddi2016fast} do not lend themselves to the fastest convergence rates for Krasulina's method. Despite its nonconvexity, however, \eqref{eqn:RayleighQuotient} has a somewhat benign optimization landscape and a whole host of algorithmic techniques and analytical tools have been developed for such structured nonconvex problems in recent years that guarantee fast convergence to a global solution. In this paper, we leverage some of these recent developments to guarantee near-optimal global convergence of two variants of Krasulina's method in the case of high-rate streaming data.

Before proceeding further, it is worth noting that while Krasulina's method primarily focuses on the 1-PCA problem, it \emph{can} be used to solve the $k$-PCA problem. But such an \emph{indirect} approach, which involves repeated use of the Krasulina's method $k$ times, can be inefficient in terms of sample complexity~\cite[Section~1]{allen2016first}. We leave investigation of a near-optimal direct method for the $k$-PCA problem involving high-rate streaming data for future work.

\subsection{Our Contributions}\label{subsec:Contributions}
In this paper, we propose and analyze two distributed variants of Krasulina's method for estimating the top eigenvector of a covariance matrix from fast streaming, independent and identically distributed (i.i.d.) data samples. Our theoretical analysis, as well as numerical experiments on synthetic and real data, establish near-optimality of the proposed algorithms. In particular, our analysis shows that the proposed algorithms can achieve the optimal convergence rate of $O(1/M)$ for 1-PCA after processing a total of $O(M)$ data samples (see \cite[Theorem~1.1]{jain2016streaming} and \cite[Theorem~6]{allen2016first}). In terms of details, following are our key contributions:

\begin{enumerate}
\item Our first contribution corresponds to the scenario in which there is a mismatch of \mbox{$N \in \Z_+ > 1$} between the data streaming rate and the processing capability of a single processor, i.e., one iteration of Krasulina's method on one processor takes as long as $N$ data arrival epochs. Our solution to this problem, which avoids discarding of samples, involves splitting the data stream into $N$ parallel streams that are then input to $N$ interconnected processors. Note that this splitting effectively reduces the streaming rate at each processor by a factor of $N$. We then propose and analyze a distributed variant of Krasulina's method---termed \DK---that solves the 1-PCA problem for this distributed setup consisting of $N$ processing nodes. Our analysis shows that \DK~can result in an improved convergence rate of $O(1/Nt)$ after $t$ iterations (Theorem~\ref{thm:FinalResult}), as opposed to the $O(1/t)$ rate for the classical Krasulina's method at any one of the nodes seen in isolation. Establishing this result involves a novel analysis of Krasulina's method that brings out the dependence of its convergence rate on the variance of the sample covariance matrix; this analysis coupled with a variance reduction argument leads to the convergence rate of $O(1/Nt)$ for \DK~under appropriate conditions.

\item Mini-batching of data samples has long been promoted as a strategy in stochastic methods to reduce the wall-clock time. Too large of a mini-batch, however, can have an adverse effect on the algorithmic performance; see, e.g., \cite[Sec.~VIII]{shamir2014distributed}. One of the challenges in mini-batched stochastic methods, therefore, is characterizing the mini-batch size that leads to near-optimal convergence rates in terms of the number of processed samples. In~\citep{agarwal2011distributed,cotter2011better,dekel:JMLR12,shamir2014distributed,ruder2016overview,golmant2018computational,goyal2017accurate}, for example, the authors have focused on this challenge for the case of mini-batch SGD for convex and nonconvex problems. In the case of nonconvex problems, however, the guarantees only hold for convergence to first-order stationary points. In contrast, our second contribution is providing a comprehensive understanding of the global convergence behavior of mini-batch Krasulina's method. In fact, our analysis of \DK~is equivalent to that of a mini-batch (centralized) Krasulina's method that uses a mini-batch of $N$ samples in each iteration. This analysis, therefore, already guarantees near-optimal convergence rate with arbitrarily high probability, as opposed to $3/4$ probability for~\citep{yang2018history}, for an appropriately mini-batched Krasulina's method in a centralized setting. In addition, in the case of high-rate streaming data that requires splitting the data stream into $N$ parallel ones, we characterize the \emph{global} convergence behavior of a mini-batch generalization of \DK---termed \DMK---in terms of the mini-batch size. This involves specifying the conditions under which mini-batches of size $B/N$ per node can lead to near-optimal convergence rate of $O(1/Bt)$ after $t$ iterations of \DMK~(Theorem~\ref{thm:DistributedMiniBatch}). An implication of this analysis is that for a fixed (network-wide) sample budget of $T$ samples, \DMK~can achieve $O(1/T)$ rate after $t := T/B$ iterations provided the (network-wide) mini-batch size $B$ satisfies $B=O(T^{1-\frac{2}{c_0}})$ for some constant $c_0>2$ (Corollary~\ref{cor:BatchNoLatency}).

\item Our next contribution is an extended analysis of \DMK~that concerns the scenario where (computational and/or communication) resource constraints translate into individual nodes still receiving more data samples than they can process in one iteration of \DMK. This resource-constrained setting necessitates \DMK~dropping $\mu \in \Z_+$ samples across the network in each iteration. Our analysis in this setting shows that such loss of samples need not result in sub-optimal performance. In particular, \DMK~can still achieve near-optimal convergence rate as a function of the number of samples arriving in the network---for both infinite-sample and finite-sample regimes---as long as $\mu= O(B)$ (Corollary~\ref{cor:latency}).
	
\item We provide numerical results involving both synthetic and real-world data to establish the usefulness of the proposed algorithms, validate our theoretical analysis, and understand the impact of the number of dropped samples per iteration of \DMK~on the convergence rate. These results in particular corroborate our findings that increasing the mini-batch size improves the performance of \DMK~up to a certain point, after which the convergence rate starts to decrease.
\end{enumerate}

\subsection{Related Work}\label{subsec:RelatedWork}
Solving the PCA problem efficiently in a number of settings has been an active area of research for decades. \citep{krasulina1969method,oja} are among the earliest and most popular methods to solve PCA in streaming data settings. Several variants of these methods have been proposed over the years, including~\citep{BinYang1995,Chatterjee2005,Doukopoulos2008}. Like earlier developments in stochastic approximation methods~\citep{robbins:AMS51}, such variants were typically shown to converge asymptotically. Convergence rate analyses for stochastic optimization in finite-sample settings~\citep{shapiro2000rate,linderoth2006empirical} paved the way for non-asymptotic convergence analysis of different variants of the stochastic PCA problem, which is fundamentally a nonconvex optimization problem. Because of the vastness of literature on (stochastic) PCA, this work is tangentially or directly related to a number of such prior works. We review some of these works in the following under the umbrellas of different problem setups, with the understanding that the resulting lists of works are necessarily incomplete. Much of our discussion in the following focuses on solving the PCA problem in (fast) streaming and distributed data settings, which is the main theme in this paper.

\textbf{Sketching for PCA.} Sketching methods have long been studied in the literature for solving problems involving matrix computations; see~\citep{woodruff2014sketching} for a review of such methods. The main idea behind these methods is to compress data using either randomized or deterministic sketches and then perform computations on the resulting low-dimensional data. While sketching has been used as a tool to solve the PCA problem in an efficient manner (see, e.g., \citep{Warmuth2007,halko2011finding,liberty2013simple,leng2015online,karnin2015online}), the resulting methods cannot be used to \emph{exactly} solve \eqref{eqn:PCA} in the fast streaming settings of this paper.

\textbf{Online PCA.} The PCA problem has also been extensively studied in online settings. While such settings also involve streaming data, the main goal in online PCA is to minimize the cumulative subspace estimation error over the entire time horizon of the algorithm. The online PCA framework, therefore, is especially useful in situations where either the underlying subspace changes over time or there is some adversarial noise in the sampling process. Some of the recent works in this direction include~\citep{garber2015online,allen2017follow,garber2018regret,marinov2018streaming,kotlowski2019bandit}.

\textbf{Stochastic convex optimization for PCA.} One approach towards solving~\eqref{eqn:ERM} in streaming settings is to relax the PCA problem to a convex optimization problem and then use SGD to solve the resulting stochastic convex optimization problem~\citep{arora2013stochastic,Garber,nie2016online}. The benefit of this approach is that now one can rely on rich literature for solving stochastic convex problems using SGD. But the tradeoff is that one now needs to store an iterate of dimension $\R^{d\times d}$, as opposed to an iterate of dimension $\R^{d\times k}$ when we solve the PCA problem in its original nonconvex form. Due to these high memory requirements of $O(d^2)$, we limit ourselves to solving PCA in the nonconvex form.

\textbf{Streaming PCA and nonconvex optimization.} The PCA problem in the presence of streaming data can also be tackled as an explicit constrained nonconvex optimization program~\citep{zhang2016global,DeSa2014}. In~\citep{zhang2016global}, for instance, the problem is solved as an optimization program over the Grassmannian manifold. The resulting analysis, however, relies on the availability of a good initial guess. In contrast, the authors in \citep{DeSa2014} analyze the use of the SGD for solving certain nonconvex problems that include PCA. The resulting approach, however, requires the step size to be a significantly small constant for eventual convergence (e.g., $10^{-12}$ for the Netflix Prize dataset); this translates into slower convergence in practice.

\textbf{Classical stochastic approximation methods for PCA.} %
Recent years have seen an increased interest in understanding the global convergence behavior of classical stochastic approximation methods such as Krasulina's method~\citep{krasulina1969method} and Oja's rule \citep{oja} for the PCA problem in non-asymptotic settings~\citep{allen2016first,Chatterjee2005,Hardt2015,Shamir,Shamir2015,jain2016streaming,pmlr-v51-li16b,tang2019exponentially,henriksen2019adaoja,amid2019implicit}. Some of these works, such as~\citep{Shamir} and~\citep{Shamir2015}, use variance reduction techniques to speed-up the algorithmic convergence. Such works, however, require multiple passes over the data, which makes them ill-suited for fast streaming settings. The analysis in~\citep{Shamir} and~\citep{Shamir2015} also requires an initialization close to the true subspace, which is somewhat unlikely in practice. Among other works, the authors in~\citep{allen2016first} provide eigengap-free convergence guarantees for Oja's rule. Since the results in this work do not take into account the variance of data samples, they do not generalize to mini-batch/distributed streaming settings. The authors in~\citep{jain2016streaming} do provide variance-dependent guarantees for Oja's rule, which makes this work the most relevant to ours. In particular, the authors in~\citep{yang2018history} have extended the initial analysis in~\citep{jain2016streaming} to mini-batch settings. However, the results derived in~\citep{yang2018history} only hold with probability $3/4$, which is in sharp contrast to the results of this paper. Note that while one could increase the probability of success in~\citep{yang2018history} through multiple algorithmic runs, this is not a feasible strategy in streaming settings.

\textbf{Distributed PCA and streaming data.} Several recent works such as \citep{balcan2016communication,boutsidis2016optimal,garber2017communication,de2017accelerated} have focused on the PCA problem in distributed settings. Among these works, the main focus in \citep{balcan2016communication,boutsidis2016optimal,garber2017communication} is on improving the communications efficiency. This is accomplished in \citep{balcan2016communication,boutsidis2016optimal} by sketching the local iterates and communicating the resulting compressed iterates to a central server in each iteration. In contrast, \cite{garber2017communication} provides a batch solution in which every node in the network first computes the top eigenvector of its local (batch) covariance matrix and then, as a last step of the algorithm, all the local eigenvector estimates are summed up at a central server to provide an eigenvector estimate for the global covariance matrix. In contrast to these works, our focus in this paper is on establishing that distributed (mini-batch) variants of stochastic approximation methods such as Oja's rule and Krasulina's method can lead to improved convergence rates, as a function of the number of samples, for the PCA problem in fast streaming settings. In this regard, our work is more closely related to~\citep{de2017accelerated}, where the authors use the momentum method to accelerate convergence of power method and further extend their work to stochastic settings. However, the approach of \citep{de2017accelerated} relies on a variance reduction technique that requires a pass over the complete dataset every once in a while; this is impractical in streaming settings. In addition, theoretical guarantees in~\citep{de2017accelerated} are based on the assumption of a ``good'' initialization; further, an implicit assumption in~\citep{de2017accelerated} is that inter-node communications is fast enough that there are no communication delays.

\textbf{Connections to stochastic nonconvex optimization.} Recent years have also seen an increased focus on understanding (variants of) SGD for general (typically unconstrained) stochastic nonconvex optimization problems. Among such works, some have focused on classical SGD~\citep{ge2015escaping,hazan2016graduated,hazan2017near}, some have studied variance-reduction variants of SGD~\citep{reddi2016stochastic,reddi2016bfast}, and some have investigated accelerated variants of stochastic nonconvex optimization~\citep{allen2018natasha,allen2018make}. In particular, works such as \citep{reddi2016stochastic,allen2016variance} are directly relevant to this paper since these works also use mini-batches to reduce sample variance and improve on SGD convergence rates. While (implicit, through the distributed framework, and explicit) mini-batching is one of the key ingredients of our work also, this paper differs from such related works because of its ability to prove convergence to a global optimum of the 1-PCA problem. In contrast, aforementioned works only provide guarantees for convergence to first-order stationary points of (typically unconstrained) stochastic nonconvex optimization problems.

\subsection{Notational Convention and Paper Organization}\label{subsec:paper_Organization}
We use lower-case ($a$), bold-faced lower-case ($\ba$), and bold-faced upper-case ($\bA$) letters to represent scalars, vectors, and matrices, respectively. Given a scalar $a$ and a vector $\ba$, $\lceil a \rceil$ denotes the smallest integer greater than or equal to $a$, while $\|\ba\|_2$ denotes the $\ell_2$-norm of $\ba$. Given a matrix $\bA$, $\|\bA\|_2$ denotes its spectral norm and $\|\bA\|_F$ denotes its Frobenius norm. In addition, assuming $\bA \in \R^{d \times d}$ to be a positive semi-definite matrix, $\lambda_i(\bA)$ denotes its $i$-th largest eigenvalue, i.e., $\|\bA\|_2 := \lambda_1(\bA) \geq \lambda_2(\bA) \geq \dots \geq \lambda_d(\bA) \geq 0$. Whenever obvious from the context, we drop $\bA$ from $\lambda_i(\bA)$ for notational compactness. Finally, $\E\{\cdot\}$ denotes the expectation operator, where the underlying probability space $(\Omega,\cF,\bbP)$ is either implicit from the context or is explicitly pointed out in the body.

The rest of this paper is organized as follows. We first provide a formal description of the problem and the system model in Section~\ref{sec:ProblemFormulation}. The two proposed variants of Krasulina's method that can be used to solve the 1-PCA problem in fast streaming settings are then presented in Section~\ref{sec:DistributedPCA}. In Section~\ref{sec:ConvergenceAnalysis}, we provide theoretical guarantees for the proposed algorithms, while proofs / outlines of the proofs of the main theoretical results are provided in Section~\ref{sec:Proofs}. Finally, numerical results using both synthetic and real-world data are presented in Section~\ref{sec:NumericalResults}, while appendices are used for detailed proofs of some of the theoretical results.

\section{Problem Formulation and System Model}\label{sec:ProblemFormulation}

Our goal is to use some variants of Krasulina's method (cf.~\eqref{eqn:Krasulina}) in order to obtain an estimate of the top eigenvector of a covariance matrix from independent and identically distributed (i.i.d.) data samples that are fast streaming into a system. The algorithms proposed in this regard and their convergence analysis rely on the following sets of assumptions concerning the data and the system.

\subsection{Data Model}\label{subsec:DataModel}
We consider a streaming data setting where a new data sample $\bx_{t'} \in \R^d$ independently drawn from an unknown distribution $\cP_{x}$ arrives at a system at each sampling time instance $t'$. We assume a uniform data arrival rate of $R_s$ samples per second and, without loss of generality, take the data arrival index $t' \geq 1$ to be an integer. We also make the following assumptions concerning our data, which aid in our convergence analysis.
\begin{enumerate}[label=$\mathrm{\bf{[A\arabic*]}}$]
\item \label{Assum:A1} \emph{(Zero-mean, norm-bounded samples)} Without loss of generality, the data samples have zero mean, i.e., $\mathbb{E}_{\cP_{x}}\{\bx_{t'}\}=0$. In addition, the data samples are almost surely bounded in norm, i.e., $\|\bx_{t'}\|_2\leq r$, where we let the bound $r\geq 1$ without loss of generality.
\item \label{Assum:A2} \emph{(Spectral gap of the covariance matrix)} The largest eigenvalue of $\bSigma:=\E_{\cP_{x}}\{\bx_{t'} \bx_{t'}^\tT\}$ is strictly greater than the second largest eigenvalue, i.e., $\lambda_1(\bSigma) > \lambda_2(\bSigma) \geq \lambda_3(\bSigma) \geq \dots \geq \lambda_d(\bSigma) \geq 0$.
\end{enumerate}
\noindent Note that both Assumptions~\ref{Assum:A1} and~\ref{Assum:A2} are standard in the literature for convergence analysis of Krasulina's method and Oja's rule (cf.~\cite{Balsubramani2015,oja,allen2016first,jain2016streaming}).

We also associate with each data sample $\bx_{t'}$ a rank-one random matrix $\bA_{t'}:=\bx_{t'} \bx_{t'}^{\tT}$, which is a trivial unbiased estimate of the population covariance matrix $\bSigma$. We then define the variance of this unbiased estimate as follows.
\begin{definition}[Variance of sample covariance matrix]\label{def:SampleVariance}
We define the variance of the sample covariance matrix $\bA_{t'}:=\bx_{t'} \bx_{t'}^{\tT}$ as follows:
	   $$\sigma^2:=\mathbb{E}_{\cP_{x}}\left\{\big\|\bA_{t'} - \bSigma\big\|_F^2\right\}.$$
\end{definition}
Note that all moments of the probability distribution $\mathcal{P}_x$ exist by virtue of the norm boundedness of $\bx_{t'}$ (cf.~Assumption~\ref{Assum:A1}). The variance $\sigma^2$ of the sample covariance matrix $\bA_{t'}$ as defined above, therefore, exists and is finite.

The two algorithms proposed in this paper, namely, \DK~and \DMK, are initialized with a random vector $\bv_0\in\R^d$ that is randomly generated over the unit sphere in $\R^d$ with respect to the uniform (Haar) measure. All analysis in this paper is with respect to the natural probability space $(\Omega,\cF,\bbP)$ given by the stochastic process $(\bv_0, \bx_1, \bx_2, \dots)$ and filtered versions of this probability space.

\subsection{System Model}\label{subsec:HighrateStreaming}
Let $R_p$ denote the number of data samples that a single processing node in the system can process in one second using an iteration of the form~\eqref{eqn:Krasulina}. The focus of this paper is on the high-rate streaming setting, which corresponds to the setup in which the data arrival rate $R_s$ is strictly greater than the data processing rate $R_p$. A naive approach to deal with this computation--streaming mismatch is to discard (per second) a fraction $\alpha := R_s/R_p$ of samples in the system. Such an approach, however, leads to an equivalent reduction in the convergence rate by $\alpha$. We pursue an alternative to this approach in the paper that involves the simultaneous use of $N\geq \lceil\alpha\rceil$ interconnected processors, each individually capable of processing $R_p$ samples per second, within the system. In particular, we advocate the use of such a network of $N$ processors in the following two manners to achieve near-optimal convergence rates (as a function of the number of samples arriving at the system) for estimates of the top eigenvector of $\bSigma$ in high-rate streaming settings.

\subsubsection{Distributed Processing Over a Network of Processors}\label{subsec:Distributed}

\begin{figure}
	\centering
	\subfigure[]
        {\includegraphics[width=0.4\columnwidth]{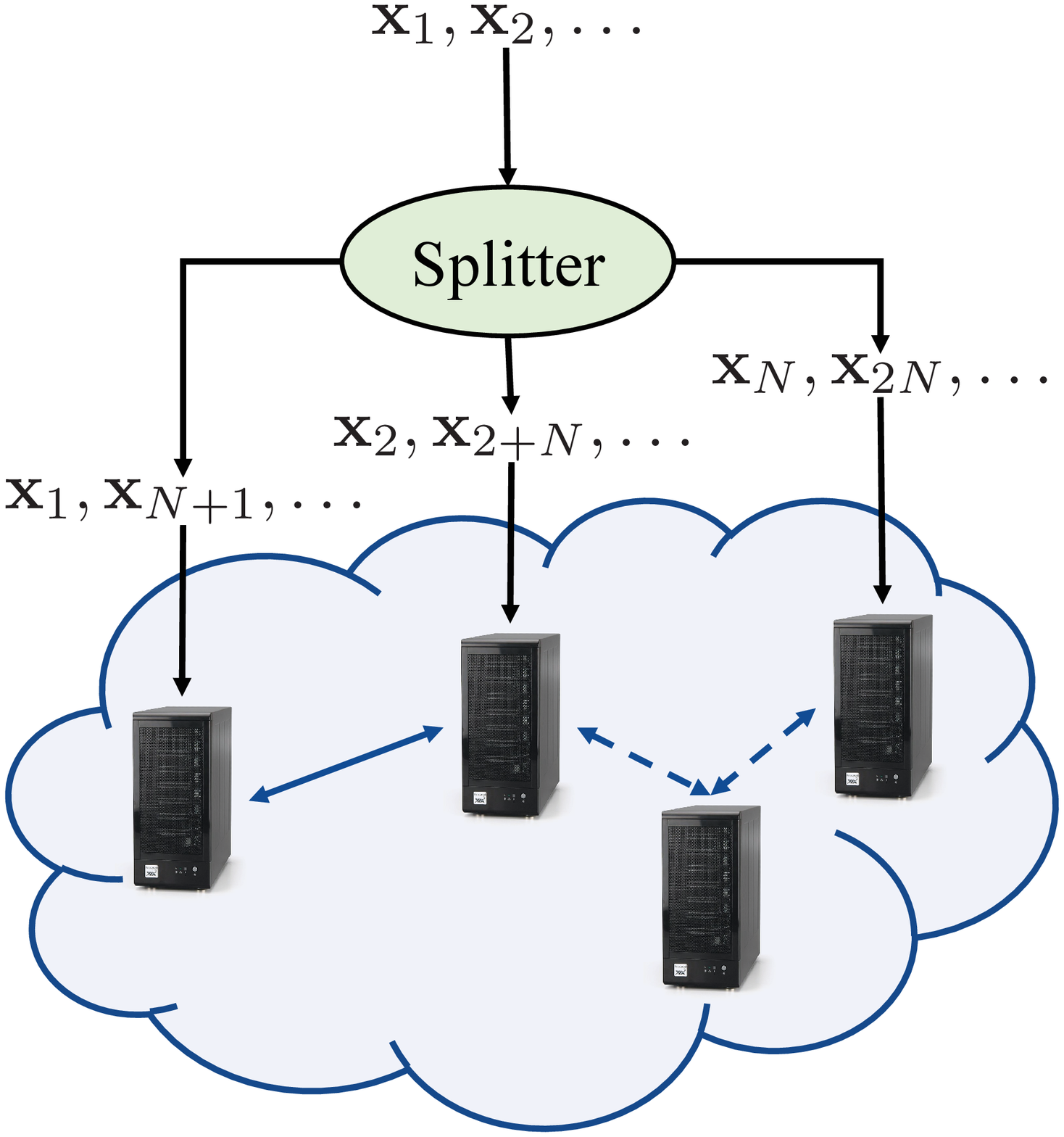}}
	\qquad
	\subfigure[]
        {\includegraphics[width=0.4\columnwidth]{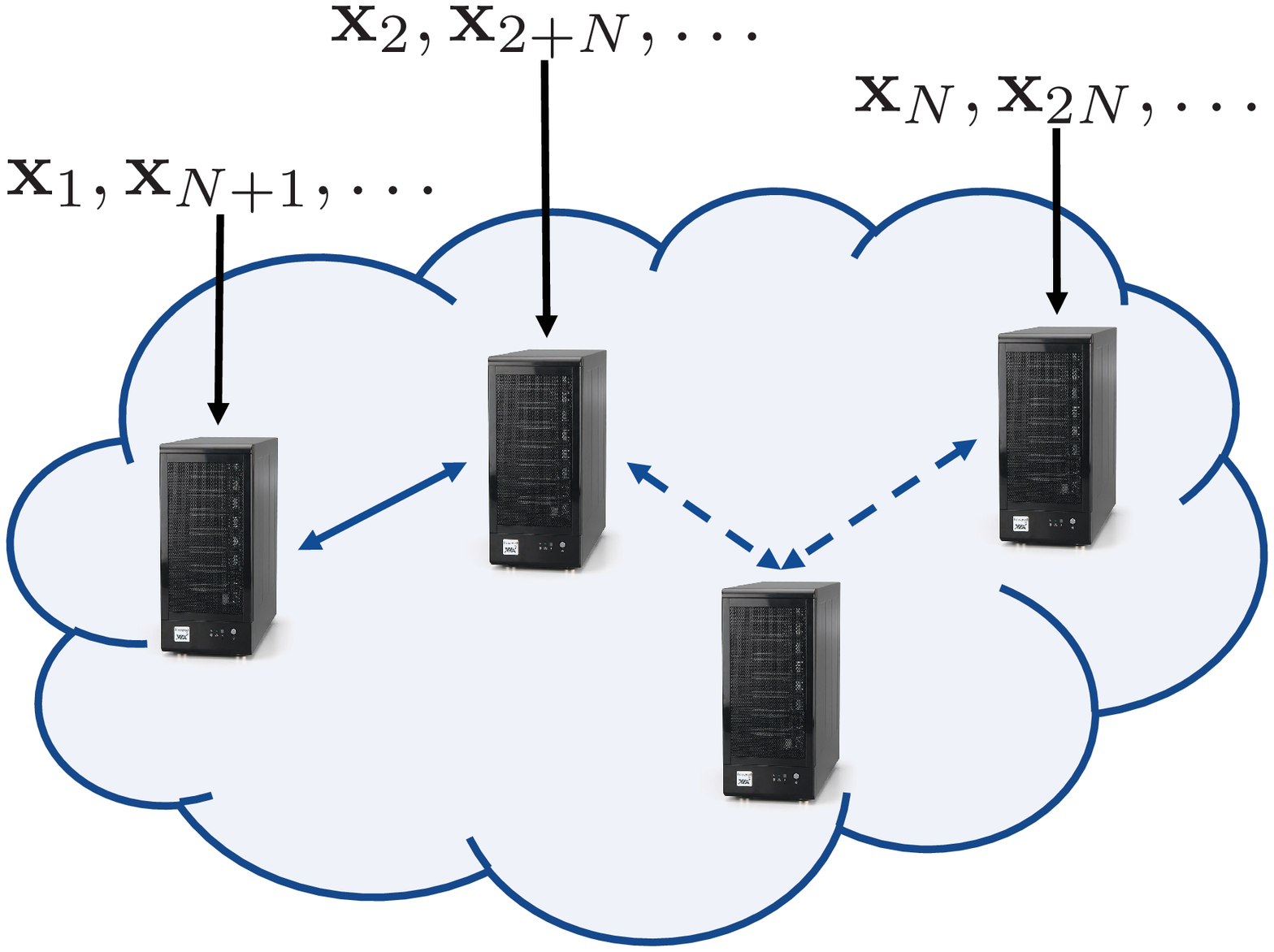}}
    \caption{The distributed PCA problem, which involves distributed processing of data over a network of $N$ processors, can
        arise in two contexts. (a) A data splitter can split a data stream into $N$ parallel streams, one for each processor in the network. In relation to the original data stream, this effectively reduces the data arrival rate for each parallel stream by a factor of $N$. (b) Data can be inherently distributed, as in the Internet-of-Things systems, and can arrive at $N$ different processing nodes as $N$ separate data streams.%
    }
	\label{fig:DistributedNetwork}
\end{figure}

We assume the fast data stream terminates into a data splitter, which splits the original stream with data rate $R_s$ samples per second into $N$ parallel streams, each with data rate $R_s/N$ samples per second, that are then synchronously input to the interconnected network of $N$ processors; see Figure~\ref{fig:DistributedNetwork}(a) for a schematic rendering of such splitting. In order to simplify notation in this setting, we reindex the data samples associated with the $i$-th processor / data stream in the following as $\{\bx_{i,t}\}_{t \in \Z_+}$, where the reindexing map $(i,t) \mapsto t'$ is simply defined as $t' = i + (t-1)N$.

We also assume the network of processors implements some message passing protocol that allows it to compute sums of locally stored vectors, i.e., $\sum_{i=1}^N \ba_i$ for the set of local vectors $\{\ba_i\}_{i=1}^N$, within the network. This could, for instance, be accomplished using either \textsf{Reduce} or \textsf{AllReduce} primitives within most message passing implementations. We let $R_c$ denote the number of these primitive (sum) operations that the message passing protocol can carry out per second in the network of $N$ processors. Note that this parameter depends upon the message passing implementation, number of nodes in the network, topology of the network, and inter-node communications bandwidth, all of which are being abstracted here through $R_c$.

Data splitting among this network of $N$ processors effectively slows down the data streaming rate at each processing node by a factor of $N$. It is under this system model that we present a distributed variant of Krasulina's method, termed \DK, in Section~\ref{ssec:DK} that operates under the assumption of $N \geq \tfrac{R_s}{R_p} + \tfrac{R_s}{R_c}$. The main analytical challenge for \DK~is understanding the scenarios under which this distributed processing over a network of processors still yields near-optimal performance; we address this challenge in Section~\ref{subsec:MainResult}.

\begin{remark}
It is straightforward to see that our developments in this paper are also applicable to the setting in which data naturally arrives in a distributed manner at $N$ different nodes, as in Figure~\ref{fig:DistributedNetwork}(b). In addition, our analysis of \DK~is equivalent to that of a mini-batch Krasulina's method running on a powerful-enough single processor that uses a mini-batch of $N$ samples in each iteration.
\end{remark}

\subsubsection{Distributed Processing Coupled with Mini-batching}\label{subsec:Minibatching}
Mini-batching in (centralized) stochastic methods, as discussed in Section~\ref{subsec:Contributions}, helps reduce the wall-clock time by reducing the number of read operations per iteration. Mini-batching of samples in distributed settings has the added advantage of reduction in the average number of primitive (sum) operations per processed sample, which further reduces the wall-clock time. It is in this vein that we put forth a mini-batched variant of \DK, which is termed \DMK, in Section~\ref{subsec:MinibatchKrasulina}.

Similar to the case of \DK~(cf.~Figure~\ref{fig:DistributedNetwork}), there are several equivalent system models that can benefit from the \DMK~framework. In keeping with our theme of fast streaming data, as well as for the sake of concreteness, we assume the system buffers (i.e., mini-batches) $B := b N \geq \lceil R_s/R_p \rceil$ samples of the incoming data stream every $B/R_s$ seconds for some parameter $b \in \Z_+$. This \emph{network-wide} mini-batch of $B$ samples is then split into $N$ parallel (local) mini-batches, each comprising $b = B/N$ samples, which are then synchronously input to the interconnected network of $N$ processors at a rate of $R_s/N$ samples per second and collaboratively processed by \DMK. In each iteration $t$ of \DMK, therefore, the network processes a total of $B \geq N$ samples, as opposed to $N$ samples for \DK. In order to simplify notation in this mini-batched distributed setting, we reindex the $b$ data samples in the mini-batch associated with the $i$-th processor in iteration $t$ of \DMK~as $\{\bx_{i,j,t}\}_{j=1,t\in\Z_+}^{j=b}$, where the reindexing map $(i,j,t) \mapsto t'$ is defined as $t'=j+(i-1)b+(t-1)B$.

The \DMK~framework can process all data samples arriving at the system as long as $N \geq \tfrac{R_s}{R_p} + \tfrac{R_s}{b R_c}$. However, when this condition is violated due to faster streaming rate $R_s$, slower processing rate $R_p$, slower summation rate $R_c$, or any combination thereof, it becomes necessary for \DMK~to discard $\mu := \left(\tfrac{b R_s}{R_p} + \tfrac{R_s}{R_c}\right) - B$ samples at the splitter per iteration. The main analytical challenges for \DMK~are, therefore, twofold: first, assuming $\mu = 0$, characterize the mini-batch size $B$ that leads to near-optimal convergence rates for \DMK~in terms of the total number of samples arriving at the system; second, when discarding of samples becomes necessary, characterize the interplay between $B$ and $\mu$ that allows \DMK~to still achieve (order-wise) near-optimal convergence rates. We address both these challenges in Section~\ref{subsec:ExtensionMiniBatch}. 
\section{Proposed Distributed Stochastic Algorithms}\label{sec:DistributedPCA}
We now formally describe the two stochastic algorithms, termed \DK~and \DMK, that can be used to solve the 1-PCA problem from high-rate streaming data under the two setups described in Section~\ref{subsec:Distributed} and Section~\ref{subsec:Minibatching}, respectively.

\subsection{Distributed Krasulina's Method (\DK) for High-rate Streaming Data}\label{ssec:DK}
Recall from the discussion in Section~\ref{subsec:Distributed} that each node $i$ in the network receives data sample $\bx_{i,t}$ in iteration $t$ of the distributed implementation, which comprises $N$ processing nodes. Unlike the centralized Krasulina's method (cf.~\eqref{eqn:Krasulina}), therefore, any distributed variant of Krasulina's method needs to process $N$ samples in every iteration $t$. Using $\bA_{i,t}$ as a shorthand for $\bx_{i,t}\bx_{i,t}^\tT$, one natural extension of \eqref{eqn:Krasulina} that processes $N$ samples in each iteration is as follows:
\begin{align}\label{eqn:MiniBatchOjaRule}
    \bv_t = \bv_{t-1} + \gamma_t {\Bigg(\frac{1}{N}\sum_{i=1}^{N}\bA_{i,t} \bv_{t-1}-\frac{1}{\|\bv_{t-1}\|_2^2}\Big(\bv_{t-1}^{\tT}\frac{1}{N}\sum_{i=1}^{N}\bA_{i,t} \bv_{t-1} \bv_{t-1}\Big)\Bigg)} = \bv_{t-1} + \gamma_t \bxi_t.
\end{align}
One natural question here is whether \eqref{eqn:MiniBatchOjaRule} can be computed within our distributed framework. The answer to this is in the affirmative under the assumption $N \geq \tfrac{R_s}{R_p} + \tfrac{R_s}{R_c}$, with the implementation (termed \DK) formally described in Algorithm~\ref{algo:Distributed_Krasulina}.

\begin{algorithm}[h]
	\textbf{Input:} Incoming data streams at $N$ processors, expressed as $\left\{\bx_{i,t} \stackrel{\text{i.i.d.}}{\sim} \cP_x\right\}_{i=1,t \in \Z_+}^N$, and a step-size sequence $\left\{\gamma_t \in \R_+\right\}_{t \in \Z_+}$\\
	\textbf{Initialize:} All processors initialize with $\bv_0\in\R^d$ randomly generated over the unit sphere%
    \vspace{-\baselineskip}%
	\algsetup{indent=1em}
	\begin{algorithmic}[1]
		\FOR{$t=1,2,\dots$,}
		\STATE \textbf{(In Parallel)} Processor $i$ receives data sample $\bx_{i,t}$ and updates $\bxi_{i,t}$ locally as follows:
		  $$\forall i \in \{1,\dots,N\}, \quad \bxi_{i,t} \gets \bx_{i,t}\bx_{i,t}^\tT \bv_{t-1}-\frac{\bv_{t-1}^{\tT}\bx_{i,t}\bx_{i,t}^\tT \bv_{t-1} \bv_{t-1}}{\|\bv_{t-1}\|_2^2}$$
		\STATE Compute $\bxi_t\leftarrow \tfrac{1}{N}\sum_{i=1}^{N}\bxi_{i,t}$ in the network using a distributed vector-sum subroutine
		\STATE Update eigenvector estimate in the network as follows: $\bv_t\leftarrow \bv_{t-1}+\gamma_t \bxi_t$
		\ENDFOR
	\end{algorithmic}
	{\bf Return:} An estimate $\bv_t$ of the eigenvector $\bq^*$ of $\bSigma$ associated with $\lambda_1(\bSigma)$
	\caption{Distributed Krasulina's Method (\DK)}
	\label{algo:Distributed_Krasulina}
\end{algorithm}

Notice that unlike classical Krasulina's method, which processes a total of $t$ samples after $t$ iterations, \DK~processes a total of $Nt$ samples after $t$ iterations in order to provide an estimate $\bv_t$ of the top eigenvector $\bq^*$ of $\bSigma$. Another natural question, therefore, is whether the estimate $\bv_t$ returned by \DK~can converge to $\bq^*$ at the near-optimal rate of $O\left(1/\text{\# of processed samples}\right)$. Convergence analysis of \DK~in Section~\ref{sec:ConvergenceAnalysis} establishes that the answer to this is also in the affirmative under appropriate conditions that are specified in Theorem~\ref{thm:FinalResult}. An important interpretation of this result is that our proposed distributed implementation of Krasulina's method can lead to linear speed-up as a function of the number of processing nodes $N$ in the network.

\subsection{Mini-batched \DK~(\DMK) for High-rate Streaming Data}\label{subsec:MinibatchKrasulina}
The distributed, mini-batched setup described in Section~\ref{subsec:Minibatching} entails each node $i$ receiving a mini-batch of $b = B/N$ data samples, $\{\bx_{i,j,t}\}_{j=1}^b$, in each iteration $t$, for a total of $B = bN$ samples across the network in every iteration. Similar to \eqref{eqn:MiniBatchOjaRule}, these $B$ samples can in principle be processed by the following variant of the original Krasulina's iteration:
\begin{align}\label{eqn:DistributedMiniBatchOjaRule}
    \bv_t = \bv_{t-1} + \gamma_t \underbrace{\Bigg(\frac{1}{B}\sum_{i=1}^{N}\sum_{j=1}^{b}\bA_{i,j,t} \bv_{t-1}-\frac{1}{\|\bv_{t-1}\|_2^2}\Big(\bv_{t-1}^{\tT}\frac{1}{B}\sum_{i=1}^{N}\sum_{j=1}^{b}\bA_{i,j,t} \bv_{t-1} \bv_{t-1}\Big)\Bigg)}_{\bxi_t},
\end{align}
where $\bA_{i,j,t}$ is a shorthand for $\bx_{i,j,t}\bx_{i,j,t}^\tT$. Practical computation of \eqref{eqn:DistributedMiniBatchOjaRule} within our distributed frame\-work, however, requires consideration of two different scenarios.
\begin{itemize}
	\item \emph{Scenario 1:} The mini-batched distributed framework satisfies $N \geq \tfrac{R_s}{R_p} + \tfrac{R_s}{b R_c}$. This enables incorporation of every sample arriving at the system into the eigenvector estimate.
    \item \emph{Scenario 2:} The mini-batched distributed framework leads to the condition $N < \tfrac{R_s}{R_p} + \tfrac{R_s}{b R_c}$. This necessitates discarding of $\mu = \left(\tfrac{b R_s}{R_p} + \tfrac{R_s}{R_c}\right) - B$ samples per iteration in the system. Stated differently, the system receives $B+\mu$ samples per iteration in this scenario, but only $B$ samples per iteration are incorporated into the eigenvector estimate.
\end{itemize}
We now formally describe the algorithm (termed \DMK) that implements \eqref{eqn:DistributedMiniBatchOjaRule} under both these scenarios in Algorithm~\ref{algo:DMB_Krasulina}.

\begin{algorithm}[h]
\textbf{Input:} Incoming streams of mini-batches $\left\{\bx_{i,j,t} \stackrel{\text{i.i.d.}}{\sim} \cP_x\right\}_{i,j=1,t \in \Z_+}^{N,b}$ at $N$ processors, size of the network-wide mini batch $B := bN$, and a step-size sequence $\left\{\gamma_t \in \R_+\right\}_{t \in \Z_+}$\\
	\textbf{Initialize:} All processors initialize with $\bv_0\in\R^d$ randomly generated over the unit sphere%
    \vspace{-\baselineskip}%
	\algsetup{indent=1em}
	\begin{algorithmic}[1]
		\FOR{$t=1,2,\dots$,}
		\STATE \textbf{(In Parallel)} $\forall i \in \{1,\dots,N\}, \quad \bxi_{i,t} \gets 0$%
            \label{algo:DMK.label.1}
		\FOR{$j=1,\dots,b$}
		\STATE \textbf{(In Parallel)} Processor $i$ receives data sample $\bx_{i,j,t}$ and updates $\bxi_{i,t}$ locally as follows:
            $$\forall i \in \{1,\dots,N\}, \quad  \bxi_{i,t} \gets \bxi_{i,t} + \bx_{i,j,t}\bx_{i,j,t}^{\tT} \bv_{t-1}-\frac{\bv_{t-1}^{\tT}\bx_{i,j,t}\bx_{i,j,t}^{\tT}\bv_{t-1} \bv_{t-1}}{\|\bv_{t-1}\|_2^2}$$
		\ENDFOR
        \STATE Compute $\bxi_t\leftarrow \tfrac{1}{B}\sum_{i=1}^{N}\bxi_{i,t}$ in the network using a distributed vector-sum subroutine
		\STATE Update eigenvector estimate in the network as follows: $\bv_t\leftarrow \bv_{t-1}+\gamma_t \bxi_t$%
            \label{algo:DMK.label.2}
        \IF{$N < \tfrac{R_s}{R_p} + \tfrac{R_s}{b R_c}$}
		\STATE The system (e.g., data splitter/buffer) receives $(B + \mu)$ additional samples during execution of Steps~\ref{algo:DMK.label.1}--\ref{algo:DMK.label.2}, out of which $\mu \in \Z_+$ samples are discarded
        \ENDIF
		\ENDFOR
	\end{algorithmic}
	{\bf Return:} An estimate $\bv_t$ of the eigenvector $\bq^*$ of $\bSigma$ associated with $\lambda_1(\bSigma)$
	\caption{Distributed Mini-batch Krasulina's Method (\DMK)}
	\label{algo:DMB_Krasulina}
\end{algorithm}

Speaking strictly in terms of implementation, the mini-batched setup of \DMK~allows one to relax the condition $N \geq \tfrac{R_s}{R_p} + \tfrac{R_s}{R_c}$ associated with \DK~to either $N \geq \tfrac{R_s}{R_p} + \tfrac{R_s}{b R_c}$, which still incorporates all samples into the eigenvector estimate, or $N < \tfrac{R_s}{R_p} + \tfrac{R_s}{b R_c}$, which involves discarding of $\mu > 0$ samples per algorithmic iteration. While this makes \DMK~particularly attractive for systems with slower communication links, the major analytical hurdle here is understanding the interplay between the different problem parameters that still allows \DMK~to achieve near-optimal convergence rates in terms of the number of samples received at the system. We tease out this interplay as part of the convergence analysis of \DMK~in Section~\ref{sec:ConvergenceAnalysis}.

\enlargethispage{0.35\baselineskip}
\section{Convergence Analysis of \DK~and \DMK}\label{sec:ConvergenceAnalysis}

Our convergence analysis of \DK~and \DMK~is based on understanding the rate at which the so-called \emph{potential function} $\Psi_t$ of these methods converges to zero as a function of the number of algorithmic iterations $t$. Formally, this potential function $\Psi_t$ is defined as follows.
\begin{definition}[Potential function]\label{def:error}
Let $\bq^*$ be the eigenvector of $\bSigma$ associated with $\lambda_1(\bSigma)$ and let $\bv_t$ be an estimate of $\bq^*$ returned by an iterative algorithm in iteration $t$. Then the quality of the estimate $\bv_{t}$ can be measured in terms of the potential function $\Psi_t : \bv_{t} \mapsto [0,1]$ that is defined as
	\begin{align}\label{eqn:Error}
	\Psi_t := 1-\frac{(\bv_{t}^{\tT}\bq^*)^2}{\|\bv_{t}\|^2}.
	\end{align}
\end{definition}

Notice that $\Psi_t$ is a measure of estimation error, which approaches $0$ as $\bv_t$ converges to any scalar multiple of $\bq^*$. This measure, which essentially computes sine squared of the angle between $\bq^*$ and $\bv_t$, is frequently used in the literature to evaluate the performance of PCA algorithms. In particular, when one initializes an algorithm with a random vector $\bv_{0}$ uniformly distributed over the unit sphere in $\R^d$ then it can be shown that $\E\{\Psi_{0}\}\leq 1-1/d$ \citep{Balsubramani2015}. While this is a statement in expectation for $t=0$, our analysis relies on establishing such a statement in probability for \emph{any} $t \geq 0$ for both \DK~and \DMK. Specifically, we show in Theorem~\ref{thm:ProbabilityBound} that $\sup_{t \geq 0} \Psi_{t} \leq 1 - O(1/d)$ with high probability as long as $\gamma_t = c/(L+t)$ for any constant $c$ and a large-enough constant $L$.

All probabilistic analysis in the following uses a filtration $(\cF_t)_{t \geq 0}$ of sub $\sigma$-algebras of $\cF$ on the sample space $\Omega$, where the $\sigma$-algebra $\cF_t$ captures the progress of the iterates of the two proposed stochastic algorithms up to iteration $t$. Mathematically, let us define the sample covariance matrix $\bA_t$ as $\bA_t := \frac{1}{N}\sum_{i=1}^{N}\bA_{i,t}$ and $\bA_t := \frac{1}{B}\sum_{i=1}^{N}\sum_{j=1}^{b}\bA_{i,j,t}$ for \DK~and \DMK, respectively. In order to simplify notation and unify some of the analysis of \DK~and \DMK, we will be resorting to the use of random matrices $\bA_t$, as opposed to $\bx_{i,t}$ and $\bx_{i,j,t}$, in the following. We then have the following definition of $\sigma$-algebras in the filtration.
\begin{definition}[$\sigma$-algebra $\cF_{t}$]\label{def:SigmaField}
The $\sigma$-algebra $\cF_{t} \subseteq \cF$ on sample space $\Omega$ for both \DK~and \DMK~is defined as the $\sigma$-algebra generated by the vector-/matrix-valued random variables $(\bv_{0}, \bA_{1},\dots,\bA_{t})$, i.e., $\cF_{t} := \sigma(\bv_{0},\bA_{1},\ldots,\bA_{t})$.
\end{definition}

In addition to the filtration $(\cF_t)_{t \geq 0}$, the forthcoming analysis also uses a sequence of nested sample spaces that is defined as follows.
\begin{definition}[Nested sample spaces]\label{def:Nested}
Let $(t_0,\epsilon_{0}), (t_1, \epsilon_1), (t_2, \epsilon_2), \ldots, (t_J,\epsilon_{J})$ be a sequence of pairs such that $0 = t_0 < t_1<t_2<\ldots<t_J$ and $\epsilon_{0}>\epsilon_1>\epsilon_2>\ldots>\epsilon_{J} > 0$ for any non-negative integer $J$. We then define a sequence $(\Omega_{t}^{'})_{t \in \Z_+}$ of nested sample spaces such that $\Omega \supset \Omega_{1}^{'} \supset \Omega_{2}^{'} \supset \ldots$, each $\Omega_t^{'}$ is $\cF_{t-1}$-measurable, and
\begin{align}\label{eqn:Nested}
	\Omega_{t}^{'} := \left\{\omega\in \Omega: \forall\,0\leq j \leq J, \sup_{t_j\leq l <t}\Psi_{l}(\omega)\leq 1-\epsilon_j\right\}.
\end{align}
\end{definition}
In words, the sample space $\Omega_{t}^{'}$ corresponds to that subset of the original sample space for which the error $\Psi_l$ in all iterations $l \in \{t_j, \dots, t-1\}$ is below $1-\epsilon_j$, where $j \in \{0,\dots,J\}$. In the following, we use the notation $\E_t\{\cdot\}$ and $\bbP_t(\cdot)$ to denote conditional expectation and conditional probability, respectively, with respect to $\Omega_{t}^{'}$.

\enlargethispage{0.3\baselineskip}

An immediate implication of Definition~\ref{def:Nested} is that, for appropriate choices of $\epsilon_j$'s, it allows us to focus on those subsets of the original sample space that ensure convergence of iterates of the proposed algorithms to the top eigenvector $\bq^*$ at the desired rates. The main challenge here is establishing that such subsets have high probability measure, i.e., $\bbP\left(\cap_{t>0}\Omega_t^{'}\right)\geq 1-\delta$ for any $\delta > 0$. We obtain such a result in Theorem~\ref{thm:IntermediateEpochProbability} in the following. We are now ready to state our main results for \DK~and \DMK.

\subsection{Convergence of \DK~(Algorithm~\ref{algo:Distributed_Krasulina})}\label{subsec:MainResult}
The first main result of this paper shows that \DK~results in linear speed-up in convergence rate as a function of the number of processing nodes, i.e., the potential function for \DK~converges to $0$ at a rate of $O(1/Nt)$. Since the system receives a total of $Nt$ samples at the end of $t$ iterations of \DK, this result establishes that \DK~is order-wise near-optimal in terms of sample complexity for the streaming PCA problem. The key to proving this result is characterizing the convergence behavior of \DK~in terms of variance of the sample covariance matrix $\bA_t := \frac{1}{N}\sum_{i=1}^{N}\bA_{i,t}$ that is implicitly computed within \DK. We denote this variance as $\sigma_N^2$, which is defined as follows.
\begin{definition}[Variance of sample covariance in \DK]\label{def:SampleVarianceDistributed}
The variance of the distributed sample covariance matrix $\bA_t$ in \DK~is defined as follows: $$\sigma_N^2:=\mathbb{E}_{\cP_{x}}\left\{\left\|\frac{1}{N}\sum_{i=1}^{N}\bx_{i,t}\bx_{i,t}^\tT - \bSigma\right\|_F^2\right\}.$$
\end{definition}
It is straightforward to see from Definition~\ref{def:SampleVariance} and Definition~\ref{def:SampleVarianceDistributed} that $\sigma_N^2 = \sigma^2/N$. This reduction in variance of the sample covariance matrix within \DK~essentially enables the linear speed-up in convergence. In terms of specifics, we have the following convergence result for \DK.
\begin{theorem}\label{thm:FinalResult}
    Fix any $\delta \in (0,1)$ and pick $c:=c_0/2(\lambda_1 - \lambda_2)$ for any $c_0 > 2$. Next, define
	\begin{align}\label{eqn:LowerboundL}
	   L_{1}:= \frac{64 edr^4 \max(1, c^2)}{\delta^2}\ln\frac{4}{\delta},\quad L_{2}:=\frac{512 e^2 d^2 \sigma_N^2\max(1, c^2)}{\delta^4}\ln\frac{4}{\delta},
	\end{align}
	pick any $L\geq L_1+L_2$, and choose the step-size sequence as $\gamma_t := c/(L+t)$. Then, as long as Assumptions~\ref{Assum:A1} and \ref{Assum:A2} hold, we have for \DK~that there exists a sequence $(\Omega_{t}^{'})_{t \in \Z_+}$ of nested sample spaces such that $\bbP\left(\cap_{t>0}\Omega_t^{'}\right)\geq 1-\delta$ and
	\begin{align}\label{eqn:FinalResult}
	\E_t\left\{\Psi_t\right\} \leq C_1\Big(\frac{L + 1}{t + L + 1}\Big)^{\frac{c_0}{2}} + C_2\Big(\frac{\sigma_N^2}{t + L + 1}\Big),
	\end{align}
    where $C_1$ and $C_2$ are constants defined as
	$$C_1 := \frac{1}{2}\Bigg(\frac{4ed}{\delta^2}\Bigg)^{\frac{5}{2\ln2}}e^{2c^2\lambda_1^2/L}\quad\textnormal{and}\quad C_2 := \frac{8 c^2 e^{(c_0+2c^2\lambda_1^2)/L}}{(c_0-2)}.$$
\end{theorem}

\begin{remark}
While we can obtain a similar result for the case of $c_0 \leq 2$, that result does not lead to any convergence speed-up. In particular, the convergence rate in that case becomes $O(t^{-c_0/2})$, which matches the one in~\citep{Balsubramani2015}.
\end{remark}

\textbf{Discussion.} A proof of Theorem~\ref{thm:FinalResult}, which is influenced by the proof technique employed in~\citep{Balsubramani2015}, is provided in Section~\ref{sec:Proofs}. Here, we discuss some of the implications of this result, especially in relation to~\citep{Balsubramani2015}. The different problem parameters affecting the performance of stochastic methods for streaming PCA include: ($i$) dimensionality of the ambient space, $d$, ($ii$) eigengap of the population covariance matrix, $(\lambda_1 - \lambda_2)$, ($iii$) upper bound on norm of the received data samples, $r$, and ($iv$) variance of the sample covariance matrix, $\sigma^2$ and/or $\sigma_N^2$. Theorem~\ref{thm:FinalResult} characterizes the dependence of \DK~on all these parameters and significantly improves on the related result provided in~\citep{Balsubramani2015}.

First, Theorem~\ref{thm:FinalResult} establishes \DK~can achieve the convergence rate $O(\sigma_N^2/t) \equiv O(\sigma^2/Nt)$ with high probability (cf.~\eqref{eqn:FinalResult}). This is in stark contrast to the result in~\citep{Balsubramani2015}, which is independent of variance of the sample covariance matrix, thereby only guaranteeing convergence rate of $O(r^4/t)$ for \DK~and its variants. This ability of variants of Krasulina's methods to achieve faster convergence through variance reduction is arguably one of the most important aspects of our analysis. Second, in comparison with~\citep{Balsubramani2015}, Theorem~\ref{thm:FinalResult} also results in an improved lower bound on choice of $L$ by splitting it into two quantities, viz., $L_1$ and $L_2$ (cf.~\eqref{eqn:LowerboundL}). This improved bound allows larger step sizes, which also results in faster convergence. In terms of specifics, $L_1$ in the theorem is on the order of $\Omega(r^4d/\delta^2)$, which is an improvement over $\Omega(r^4 d^2/\delta^4)$ bound of~\citep{Balsubramani2015}. On the other hand, while $L_2$ has same dependence on $\delta$ and $d$ as~\citep{Balsubramani2015}, it depends on $\sigma_N^2$ instead of $r^4$ and, therefore, it reduces with an increase in $N$. Third, the improved lower bound on $L$ also allows for an improved dependence on the dimensionality $d$ of the problem. Specifically, for large enough $t$ and $N$, the dependence on $d$ in \eqref{eqn:FinalResult} is due to the higher-order (first) term and is of the order $O(d^{\frac{5}{2\ln2} + \frac{c_0}{2}})$, as opposed to $O(d^{\frac{5}{2\ln2} + c_0})$ for~\citep{Balsubramani2015}. It is worth noting here, however, that this is still loser than the result in~\cite{jain2016streaming} that has only $\log^2(d)$ dependence on $d$ in higher-order error terms. Finally, in terms of the eigengap, our analysis has optimal dependence of $1/(\lambda_1-\lambda_2)^2$, which also matches the dependence in~\cite{Balsubramani2015}. We conclude by noting that this dependence of the performance of \DK~on different problem parameters is further highlighted through numerical experiments in Section~\ref{sec:NumericalResults}.

\begin{remark}
While Theorem~\ref{thm:FinalResult} is for (a distributed variant of) Krasulina's method, Oja's rule can also be analyzed using similar techniques; see, e.g., the discussion in~\citep{Balsubramani2015}.
\end{remark}

\subsection{Convergence of \DMK~(Algorithm~\ref{algo:DMB_Krasulina})}\label{subsec:ExtensionMiniBatch}

The convergence analysis of \DMK~follows from slight modifications of the proof of Theorem~\ref{thm:FinalResult} for \DK. The final set of results, which covers the two scenarios of zero data loss ($\mu = 0$) and some data loss ($\mu > 0$) in each iteration, is characterized in terms of variance of the (mini-batched) sample covariance $\bA_t:=\frac{1}{B}\sum_{i=1}^{N}\sum_{j=1}^{b}\bA_{i,j,t}$ associated with \DMK.
\begin{definition}[Variance of sample covariance in \DMK]\label{def:SampleVarianceMinibatch}
The variance of the distributed sample covariance matrix $\bA_t$ in \DMK~is defined as follows: $$\sigma_B^2:=\mathbb{E}_{\cP_{x}}\left\{\left\|\frac{1}{B}\sum_{i=1}^{N}\sum_{j=1}^{b}\bx_{i,j,t}\bx_{i,j,t}^\tT - \bSigma\right\|_F^2\right\}.$$
\end{definition}
It is once again straightforward to see that $\sigma_B^2 = \sigma^2/B$. We now split our discussion of the convergence of \DMK~according to the two scenarios discussed in Section~\ref{subsec:MinibatchKrasulina}.

\subsubsection{Scenario 1---\DMK~with no data loss: $N \geq \frac{R_s}{R_p} + \frac{R_s}{b R_c} \Longrightarrow \mu = 0$}
Analytically, this scenario is similar to \DK, with the only difference being that we are now incorporating an average of $B$ sample covariances $\bx_{i,j,t}\bx_{i,j,t}^\tT$ in the estimate in each iteration (as opposed to $N$ sample covariances for \DK). We therefore have the following generalization of Theorem~\ref{thm:FinalResult} in this scenario.
\begin{theorem}\label{thm:DistributedMiniBatch}
Let the parameters and constants be as specified in Theorem~\ref{thm:FinalResult}, except that the parameter $L_2$ is now defined as $L_2 := \frac{512 e^2 d^2 \sigma_B^2\max(1, c^2)}{\delta^4}\ln\frac{4}{\delta}$. Then, as long as Assumptions~\ref{Assum:A1} and \ref{Assum:A2} hold, we have for \DMK~that $\bbP\left(\cap_{t>0}\Omega_t^{'}\right)\geq 1-\delta$ and  	
	\begin{align}\label{eqn:DMK.FinalResult}
	\E_t\left\{\Psi_t\right\} \leq C_1\Big(\frac{L + 1}{t + L + 1}\Big)^{\frac{c_0}{2}} + C_2\Big(\frac{\sigma_B^2}{t + L + 1}\Big).
	\end{align}
\end{theorem}

The proof of this theorem can be obtained from that of Theorem~\ref{thm:FinalResult} by replacing $1/N$ and $\sigma_N^2$ in there with $1/B$ and $\sigma_B^2$, respectively. Similar to the case of \DK, this theorem establishes that \DMK~can also achieve linear speed-up in convergence as a function of the network-wide mini-batch size $B$ with very high probability, i.e., $\E_t\left\{\Psi_t\right\} = O(\sigma_B^2/t) \equiv O(\sigma^2/Bt)$.

Our discussions of \DK~and \DMK~have so far been focused on the infinite-sample regime, in which the number of algorithmic iterations $t$ for both algorithms can grow unbounded. We now focus on the implications of our results for the finite-sample regime, in which a final estimate is produced at the end of arrival of a total of $T \ggg 1$ samples.\footnote{An implicit assumption here is that $T$ is large enough that it precludes the use of a batch PCA algorithm.} This finite-sample regime leads to an interesting interplay between $N$ (resp., $B)$ and the total number of samples $T$ for linear speed-up of \DK~(resp., \DMK). We describe this interplay in the following for \DMK; the corresponding result for \DK~follows by simply replacing $B$ with $N$ in this result.
\begin{corollary}\label{cor:BatchNoLatency}
Let the parameters and constants be as specified in Theorem~\ref{thm:DistributedMiniBatch}. Next, pick parameters $(L_1', L_2')$ such that $L_1 ' \geq L_1$ and $L_2' \geq L_2/\sigma_B^2$, and define the final number of algorithmic iterations for \DMK~as $T_B :=T/B$. Then, as long as Assumptions~\ref{Assum:A1} and \ref{Assum:A2} hold and the network-wide mini-batch size satisfies $B \leq T^{1-\tfrac{2}{c_0}}$, we have that $\bbP\left(\cap_{t>0}\Omega_t^{'}\right)\geq 1-\delta$ and
    \begin{align}
	   \E_{T_B}\left\{\Psi_{T_B}\right\} \leq c_0 C_1 \frac{{L_1'}^{c_0/2}}{T} + c_0 C_1 \Bigg(\frac{\sigma^2 L_2'}{T}\Bigg)^{c_0/2} + \frac{C_2 \sigma^2}{T}.
	\end{align}
\end{corollary}
\begin{proof}
	Substituting $t = T_B$ in \eqref{eqn:DMK.FinalResult} and using simple upper bounds yield
	\begin{align*}
	   \E_{T_B}\left\{\Psi_{T_B}\right\} \leq C_1\Big(\frac{L + 1}{L + T_B}\Big)^{\frac{c_0}{2}} + C_2\Big(\frac{ \sigma_B^2}{T_B}\Big) \leq 2 C_1\Big(\frac{L}{T_B}\Big)^{\frac{c_0}{2}} + C_2\Big(\frac{\sigma_B^2}{T_B}\Big).
	\end{align*}
	Next, substituting $L=L_1'+\sigma_B^2 L_2'$ in this expression gives us
	\begin{align}\label{eqn:E_TB}
	   \E_{T_B}\left\{\Psi_{T_B}\right\}\leq c_0 C_1\Big(\frac{L_1'}{T_B}\Big)^{\frac{c_0}{2}} + c_0 C_1\Big(\frac{\sigma_B^2 L_2'}{T_B}\Big)^{\frac{c_0}{2}} + C_2\Big(\frac{\sigma_B^2}{T_B}\Big).
	\end{align}
    Since $\sigma_B^2 = \sigma^2/B$ and $T_B = T/B$, \eqref{eqn:E_TB} reduces to the following expression:
    \begin{align*}
	   \E_{T_B}\left\{\Psi_{T_B}\right\} \leq c_0 C_1 \Bigg(\frac{B L_1'}{T}\Bigg)^{c_0/2}+c_0 C_1 \Bigg(\frac{ \sigma^2 L_2'}{T}\Bigg)^{c_0/2}+\frac{C_2 \sigma^2}{T}.
	\end{align*}
    The proof now follows from the assumption that $B \leq T^{1-\tfrac{2}{c_0}}$.
\end{proof}

\textbf{Discussion.} Corollary~\ref{cor:BatchNoLatency} dictates that linear convergence speed-up for \DMK~(resp., \DK) occurs in the finite-sample regime provided the network-wide mini-batch size $B$ (resp., number of processing nodes $N$) scales sublinearly with the total number of samples $T$. In particular, the proposed algorithms achieve the best (order-wise) convergence rate of $O(1/T)$ for appropriate choices of system parameters. We also corroborate this theoretical finding with numerical experiments involving synthetic and real-world data in Section~\ref{sec:NumericalResults}.

\subsubsection{Scenario 2---\DMK~with data loss: $N < \frac{R_s}{R_p} + \frac{R_s}{b R_c} \Longrightarrow \mu > 0$}

The statement of Theorem~\ref{thm:DistributedMiniBatch} for \DMK~in the lossless setting immediately carries over to the resource-constrained setting that causes loss of $\mu~(> 0)$ samples per iteration. The implication of this result is that \DMK~can achieve convergence rate of $O(1/Bt)$ in the infinite-sample regime after receiving a total of $(B+\mu)t$ samples. Therefore, it trivially follows that \DMK~can achieve order-wise near-optimal convergence rate in the infinite-sample regime as long as $\mu = O(B)$.

We now turn our attention to understanding the interplay between $\mu$, $B$, and the total number of samples $T$ arriving at the system for the resource-constrained finite-sample setting for \DMK. To this end, we have the following generalization of Corollary~\ref{cor:BatchNoLatency}.
\begin{corollary}\label{cor:latency}
Let the parameters and constants be as specified in Corollary~\ref{cor:latency}, and define the final number of algorithmic iterations for \DMK~as $T_B^\mu :=T/(B+\mu)$. Then, as long as Assumptions~\ref{Assum:A1} and \ref{Assum:A2} hold, we have that $\bbP\left(\cap_{t>0}\Omega_t^{'}\right)\geq 1-\delta$ and
\begin{align}\label{eqn:BoundLatency}
	\E_{T_B^\mu}\left\{\Psi_{T_B^\mu}\right\} \leq c_0 C_1 \Bigg(\frac{(B+\mu)L_1'}{T}\Bigg)^{c_0/2}+c_0 C_1 \Bigg(\frac{(B+\mu) \sigma^2 L_2'}{BT}\Bigg)^{c_0/2}+\frac{C_2 \sigma^2(B+\mu)}{BT}.
	\end{align}
\end{corollary}
\begin{proof}
  The proof of this corollary follows from replacing $T_B$ with $T_B^\mu$ in \eqref{eqn:E_TB} and subsequently substituting the values of $T_B^\mu$ and $\sigma_B^2$ in there.
\end{proof}

\textbf{Discussion.} Recall that since the distributed framework receives a total of $T$ samples, it is desirable to achieve convergence rate of $O(1/T)$. It can be seen from Corollary~\ref{cor:latency} that the first and the third terms in \eqref{eqn:BoundLatency} are the ones that dictate whether \DMK~can achieve the (order-wise) optimal rate of $O(1/T)$. To this end, the first term in \eqref{eqn:BoundLatency} imposes the condition $(B+\mu) \leq T^{1-2/c_0}$, i.e., the total number of samples received at the system (both processed and discarded) per iteration must scale sublinearly with the final number of samples $T$. In addition, the third term in \eqref{eqn:BoundLatency} imposes the condition $\mu = O(B)$, i.e., the number of samples discarded by the system in each iteration must scale no faster than the number of samples processed by the system in each iteration. Once these two conditions are satisfied, Corollary~\ref{cor:latency} guarantees near-optimal convergence for \DMK.

\section{Proof of the Main Result}\label{sec:Proofs}
The main result of this paper is given by Theorem~\ref{thm:FinalResult}, which can then be applied to any algorithm that (implicitly or explicitly) involves an iteration of the form \eqref{eqn:MiniBatchOjaRule}. We develop a proof of this result in this section, which consists of characterizing the behavior of \DK~in three different algorithmic epochs. The main result concerning the \emph{initial epoch} is described in terms of Theorem~\ref{thm:ProbabilityBound} in the following, the behavior of the \emph{intermediate epoch}, which comprises multiple \emph{sub-epochs}, is described through Theorem~\ref{thm:IntermediateEpochProbability}, while the behavior of \DK~in the \emph{final epoch} is captured through a formal proof of Theorem~\ref{thm:FinalResult} at the end of this section.

Before proceeding, recall that our result requires the existence of a sequence $(\Omega_{t}^{'})_{t \in \Z_+}$ of nested sample spaces that are defined in terms of a sequence of pairs $(t_0 \equiv 0,\epsilon_{0}), (t_1, \epsilon_1), \ldots, (t_J,\epsilon_{J})$. Our analysis of the initial epoch involves showing that for the step size $\gamma_t$ chosen as in Theorem~\ref{thm:FinalResult}, the error for all $t\geq 0$ will be less than $(1-\epsilon_0)$ with high probability for some constant $\epsilon_0$. We then define the remaining $\epsilon_j$'s as $\epsilon_j = 2^j \epsilon_0, j=1,\dots,J$, where $J$ is defined as the smallest integer satisfying $\epsilon_J\geq 1/2$. Our analysis in the intermediate epoch then focuses on establishing lower bounds on the number of iterations $t_j$ for which \DK~is guaranteed to have the error less than $1 - \epsilon_j$ with high probability. Stated differently, the intermediate epoch characterizes the sub-epochs $\{1+t_{j-1},t_j\}$ during which the error is guaranteed to decrease from $\left(1-\epsilon_{j-1}\right)$ to $\left(1-\epsilon_j\right)$ with high probability.

\subsection{Initial Epoch}\label{subsec:InitialPhase}
Our goal for the initial epoch is to show that if we pick the step size appropriately, i.e.,  we set $L$ to be large enough (cf. \eqref{eqn:LowerboundL}), then the error, $\Psi_{t}$, will not exceed a certain value with high probability. This is formally stated in the following result.
\begin{theorem}\label{thm:ProbabilityBound}
Fix any $\delta \in (0,1)$, define $\epsilon \in (0,1)$ as $\epsilon:=\delta^2/8e$, and let
    \begin{align}\label{eqn:LowerBoundL_2}
        L\geq \frac{8 dr^4 \max(1, c^2)}{\epsilon}\ln\frac{4}{\delta}+\frac{8 d^2 \sigma_N^2 \max(1, c^2)}{\epsilon^2}\ln\frac{4}{\delta}.
    \end{align}
Then, if Assumptions~\ref{Assum:A1} and \ref{Assum:A2} hold and we choose step size to be $\gamma_t=c/(L+t)$, we have
\begin{align}\label{eqn:ProbabilityBound}
\bbP\Big(\sup_{t\geq 0} \Psi_{t} \geq 1 - \frac{\epsilon}{d}\Big)\leq \sqrt{2\text{e} \epsilon}\equiv\frac{\delta}{2}.
\end{align}
\end{theorem}
In order to prove Theorem~\ref{thm:ProbabilityBound} we need several helping lemmas that are stated in the following. We only provide lemma statements in this section and move the proofs to Appendix~\ref{secApp:App_InitialPhase}. We start by writing the recursion of error metric $\Psi_t$ in the following lemma.
\begin{lemma}\label{le:ErrorReccursion}
Defining a scalar random variable
\begin{align}
z_t := 2\gamma_t \frac{(\bv_{t-1}^{\tT}\bq^*)(\bxi_t^{\tT}\bq^*)}{\|\bv_{t-1}\|_2^2},
\end{align}
we get the following recursion:
\begin{enumerate}
\item[(i)] $\Psi_t \leq \Psi_{t-1} + 4\gamma_t^2 \Big(\Big\|\frac{1}{N}\sum_{i=1}^{N}{\bA_{i,t}} - \bSigma\Big\|_F^2 + \lambda_1^2\Psi_{t-1}\Big) - z_{t}$, and
\item[(ii)] $\Psi_t \leq \Psi_{t-1} + \gamma_t^2 r^4 - z_{t}.$
\end{enumerate}
\end{lemma}
\begin{proof}
See Appendix~\ref{secApp:App_InitialPhase_1}.	
\end{proof}
Part~($i$) of this lemma will be used to analyze the algorithm in the final epoch for proof of Theorem~\ref{thm:FinalResult}, while Part~($ii$) will be used to prove Theorem~\ref{thm:ProbabilityBound} for this initial epoch and Theorem~\ref{thm:IntermediateEpochProbability} for the intermediate phase.

Next we will bound the moment generating function of $\Psi_{t}$ conditioned on $\cF_{t-1}$ (Definition~\ref{def:SigmaField}). For this, we need an upper bound on conditional variance of $z_t$, which is given below.
\begin{lemma}\label{le:Variance_zt}
The conditional variance of the random variable $z_t$ is given by
\begin{align}
\E\{(z_t - \E\{z_t\})^2|\cF_{t-1}\} \leq 16\gamma_t^2 \sigma_N^2.
\end{align}
\end{lemma}
\begin{proof}
	See Appendix~\ref{secApp:App_InitialPhase_2}.	
\end{proof}

Using this upper bound on conditional variance of $z_t$ we can now upper bound the conditional moment generating function of $\Psi_t$. In order to simplify notation, much of our discussion in the following will revolve around the moment generating function with parameter $s \in \bbS := \big\{d/4\epsilon, (2/\epsilon_0)\ln(4/\delta)\big\}$. Note, however, that similar results can be derived for any positive-valued parameter $s \in \R$.
\begin{lemma}\label{le:MGF_Bound}
The conditional moment generating function of $\Psi_t$ for $s \in \bbS$ is upper bounded as
\begin{align}\label{eqn:MGF_Bound}
\mathbb{E}\{\exp(s\Psi_t)|\cF_{t-1}\}\leq \exp\Bigg(s\Psi_{t-1} - s\E\{z_t|\cF_{t-1}\} +s\gamma_t^2 r^4 +s^2\gamma_t^2 \sigma_N^2\Bigg).
\end{align}
\end{lemma}
\begin{proof}
	See Appendix~\ref{secApp:App_InitialPhase_3}.	
\end{proof}
Note that this result is similar to~\citep[Lemma~2.3]{Balsubramani2015} with the difference being that the last term here is sample variance, $\sigma_N^2$, as opposed to upper bound on input $\|\bx_{t'}\|_2\leq r$ in~\citep[Lemma~2.3]{Balsubramani2015}. This difference prompts changes in next steps of the analysis of \DK~and it also enables us to characterize improvements in convergence rate of Krasulina's method using iterations of the form \eqref{eqn:MiniBatchOjaRule}.

We are now ready to prove the statement of Theorem~\ref{thm:ProbabilityBound}, which is based on Lemma~\ref{le:ErrorReccursion} and~\ref{le:MGF_Bound}.
\begin{proof}[Proof of Theorem~\ref{thm:ProbabilityBound}] We start by constructing a supermartingale from sequence of errors $\Psi_t$. First, restricting ourselves to $s \in \bbS$, we define quantities
$$\beta_t:=\gamma_t^2 r^4,\quad\zeta_t:=s \gamma_t^2 \sigma_N^2 ,\quad \tau_t:=\sum_{l>t}{(\beta_l + \zeta_l)},\quad\text{and}\quad M_t:=\exp{(s\Psi_{t}+s\tau_t)}.$$
Now, taking expectation of $M_t$ conditioned on the filtration $\cF_{t-1}$ we get
\begin{align*}
\E\{M_t|\cF_{t-1}\} &= \E\{\exp{(s\Psi_{t})}|\cF_{t-1}\}\exp{(s\tau_t)} \stackrel{(a)}\leq \exp{(s\Psi_{t-1}+s\beta_t + s\zeta_t+ s\tau_t)}\\
&= \exp{(s\Psi_{t-1}+s\tau_{t-1})} = M_{t-1}.
\end{align*}
Here, ($a$) is due to Lemma~\ref{le:MGF_Bound} and using the fact that $\E\{z_t|\cF_{t-1}\}\geq 0$ \citep[Theorem~2.1]{Balsubramani2015}. These calculations show that the sequence $\{M_t\}$ forms a supermartingale. Using sequence $M_t$, we can now use Doob's martingale inequality \cite[pg. 231]{durrett2010probability} to show that $\Psi_t$ will be bounded away from 1 with high probability. Specifically, for any $\Delta \in (0,1)$, we have
\begin{align*}
\bbP\Big(\sup_{t\geq 0} \Psi_{t}\geq \Delta\Big)&\leq \bbP\Big(\sup_{t\geq 0} \Psi_{t}+\tau_t\geq \Delta\Big) =\bbP\Big(\sup_{t\geq 0} \exp{(s\Psi_{t}+s\tau_t)}\geq e^{s\Delta}\Big)\\
&=\bbP\Big(\sup_{t\geq 0} M_{t}\geq e^{s\Delta}\Big) \leq \frac{\E\{M_{t_0}\}}{e^{s\Delta}}=\exp{(-s(\Delta - \tau_{0}))}\E\{e^{s\Psi_{0}}\}.
\end{align*}
Substituting $\Delta=1-\epsilon/d$ and using \cite[Lemma~2.5]{Balsubramani2015} to bound $\E{e^{s\Psi_{0}}}$ we get
\begin{align}\label{eqn:Pr_e1}
\bbP\Big(\sup_{t\geq 0} \Psi_{t}\geq 1-\frac{\epsilon}{d}\Big) \leq \exp{(-s(1-(\epsilon/d) - \tau_{0}))}e^s\sqrt{\frac{d}{2s}}.
\end{align}

Next we need to bound $\sum_{l>0}\beta_l$ and $\sum_{l>0}\zeta_l$. First we get
\begin{align}\label{eqn:sum_beta}
\sum_{l>0}{\beta_l}&=\sum_{l>0}{\gamma_l^2 r^4}=  r^4\sum_{l>0}{\gamma_l^2}= r^4\sum_{l>0}{\frac{c^2}{(l+L)^2}}\leq \frac{r^4 c^2}{L}.
\end{align}
Again using a similar procedure we get
\begin{align}\label{eqn:sum_zeta}
\sum_{l>0}{\zeta_l}\leq \frac{s \sigma_N^2 c^2}{L}.
\end{align}
Combining \eqref{eqn:sum_beta} and \eqref{eqn:sum_zeta}, along with the definition of $\tau_t$ at the beginning, we get
\begin{align}\label{eqn:tauUpperBound}
\tau_{0}\leq \frac{c^2}{L}\Bigg(r^4 + s\sigma_N^2\Bigg).
\end{align}
Now using the lower bound on $L$, we get $\tau_{0}\leq \epsilon/d$ for $s=d/4\epsilon$ as shown in Proposition~\ref{prop:tau_t0} in Appendix~\ref{sec:App_OtherResults}. Substituting this in \eqref{eqn:Pr_e1} we get
\begin{align*}
\bbP\Big(\sup_{t\geq 0} \Psi_{t}\geq 1-\frac{\epsilon}{d}\Big) \leq \exp{(-s(1-\epsilon/d - \epsilon/d))}e^s\sqrt{\frac{d}{2s}}=\exp{(2s\epsilon/d)}\sqrt{\frac{d}{2s}}.
\end{align*}
Finally, substituting $s=d/4\epsilon$, we get $\bbP\Big(\sup_{t\geq 0} \Psi_{t}\geq 1-\frac{\epsilon}{d}\Big)\leq \sqrt{2e\epsilon}.$
\end{proof}

\subsection{Intermediate Epoch}\label{subsec:Intermediate}
In Theorem~\ref{thm:ProbabilityBound} we have shown that if we choose $L$ such that it satisfies the lower bound given in Theorem~\ref{thm:ProbabilityBound} then we have error $\Psi_t$ greater than $1 - \epsilon_{0}$ (here, $\epsilon_{0}=\delta^2/8ed$) with probability $\delta$. Next, our aim is to show that if we perform enough iterations $t_J$ of \DK~then for any $t\geq t_J$ the error in the iterate will be bounded by $\Psi_t\leq 1/2$ with high probability. In order to prove this, we divide our analysis into different sub-epochs that are indexed by $j\in\{1,\dots,J\}$. Starting from $1-\epsilon_{0}$, we provide a lower bound on the number of iterations $t_j$ such that we progressively increase $\epsilon_{j}$ in each sub-epoch until we reach $\epsilon_{J}$.
\begin{theorem}\label{thm:IntermediateEpochProbability}
    Fix any $\delta \in (0,1)$ and pick $c:=c_0/2(\lambda_1 - \lambda_2)$ for any $c_0 > 2$. Next, let the number of processing nodes $N>1$, the parameter $L\geq \frac{8 r^4 \max(1, c^2)}{\epsilon_0}\ln\frac{4}{\delta}+\frac{8 \sigma_N^2 \max(1, c^2)}{\epsilon_0^2}\ln\frac{4}{\delta},$ and the step size $\gamma_t := c/(L+t)$. Finally, select a schedule $(0, \epsilon_0), (t_1, \epsilon_1),\dots,(t_J, \epsilon_J)$ such that the following conditions are satisfied:
	\begin{enumerate}[label=$\mathrm{\bf{[C\arabic*]}}$]
		\item \label{Condition:C1} $\epsilon_0=\frac{\delta^2}{8 e d}$, $\frac{3}{2}\epsilon_j\leq \epsilon_{j+1}\leq 2\epsilon_j$ for $0\leq j<J$, and $\epsilon_{J-1}\leq \frac{1}{4}$, and
		\item \label{Condition:C2} $\Big(t_{j+1} + L + 1\Big)\geq e^{5/c_0}\Big(t_j + L + 1\Big)$ for $0\leq j < J$.
	\end{enumerate}
	Then $\bbP\left(\cap_{t>0}\Omega_t^{'}\right)\geq 1-\delta$.
\end{theorem}

In order to prove this theorem, we need Lemmas~\ref{le:le4}--\ref{le:Probability_tj}, which are stated as follows.
\begin{lemma}\label{le:le4}
	For $t>t_j$, the moment generating function of $\Psi_t$ for $s \in \bbS$ conditioned on $\Omega_{t}^{'}$ satisfies
\begin{align*}
\E_t\Big\{e^{s\Psi_t}\Big\}\leq \exp\Bigg(s\Bigg(\Psi_{t-1}\Big(1 - \frac{c_0 \epsilon_j}{t+L}\Big) + \frac{c^2 r^4}{(t+L)^2} + \frac{s c^2 \sigma_N^2}{(t+L)^2}\Bigg)\Bigg).
\end{align*}
\end{lemma}
\begin{proof}
See Appendix~\ref{sec:App_Intermediate_1}.
\end{proof}

\begin{lemma}\label{le:MGF_n}
For $t>t_j$ and $s \in \bbS$, we have
\begin{align}\label{eqn:MGF_n}
\E_t\{e^{s\Psi_t}\}\leq \exp{\Bigg(s(1-\epsilon_j)\Bigg(\frac{t_j + L + 1}{t + L + 1}\Bigg)^{c_0 \epsilon_j}+ \Bigg(sc^2 r^4+s^2 c^2 \sigma_N^2 \Bigg)\Bigg(\frac{1}{t_j + L}-\frac{1}{t + L}\Bigg)\Bigg)}.
\end{align}
\end{lemma}
\begin{proof}
	See Appendix~\ref{sec:App_Intermediate_2}.	
\end{proof}

Using Lemma~\ref{le:MGF_n}, our next result deals with a specific value of $t$, namely, $t=t_{j+1}$.
\begin{lemma}\label{le:MGF_nj1}
Suppose Conditions~\ref{Condition:C1}--\ref{Condition:C2} are satisfied. Then for $0\leq j < J$ and $s \in \bbS$, we get
$$\E_{t_{j+1}}\big\{e^{s\Psi_{t_{j+1}}}\big\}\leq \exp{\Bigg(s(1-\epsilon_{j+1}) - s\epsilon_j + \Big(s c^2 r^4+s^2 c^2 \sigma_N^2\Big)\Big(\frac{1}{t_j + L}-\frac{1}{t_{j+1} + L}\Big)\Bigg)}.$$
\end{lemma}
\begin{proof}
	See Appendix~\ref{sec:App_Intermediate_3}.
\end{proof}

\begin{lemma}\label{le:Probability_tj}
Suppose Conditions \ref{Condition:C1}--\ref{Condition:C2} are satisfied. Then picking any $0<\delta<1$, we have
$$\sum_{j=1}^{J}{\bbP_{t_j}\Big(\sup_{t\geq t_j} \Psi_t > 1 - \epsilon_j\Big)}\leq \frac{\delta}{2}.$$
\end{lemma}
\begin{proof}
	See Appendix~\ref{sec:App_Intermediate_4}.
\end{proof}

\begin{proof}\emph{(Proof of Theorem~\ref{thm:IntermediateEpochProbability})}
Using results from Lemma \ref{le:Probability_tj} and Theorem~\ref{thm:ProbabilityBound} and applying union bound, we get the statement of Theorem~\ref{thm:IntermediateEpochProbability}.
\end{proof}

\subsection{Final Epoch}\label{subsec:FinalEpoch}
Now that we have shown that $\Psi_{t}\leq 1/2$ with probability $1-\delta$ for all $t \geq t_J$, we characterize in the final epoch how $\Psi_{t}$ decreases further as a function of algorithmic iterations. The following result captures the rate at which $\Psi_t$ decreases during this final epoch.
\begin{lemma}\label{le:FinalRate}
For any $t>t_J$, the (conditional) expected error in $\Psi_{t}$ is given by
$$\E_t\{\Psi_t\}\leq \Bigg(1+\frac{c_0^2 \lambda_1^2}{2(t+L)^2 (\lambda_1 - \lambda_2)^2}-\frac{c_0}{2 (t+L)}\Bigg)\E_{t-1}\{\Psi_{t-1}\}+\frac{4c^2 \sigma_N^2}{(t + L)^2}.$$
\end{lemma}
\begin{proof}
See Appendix~\ref{sec:App_FinalEpoch}.	
\end{proof}

We are now ready to prove our main result, which is given by Theorem~\ref{thm:FinalResult}.
\begin{proof}({\em Proof of Theorem \ref{thm:FinalResult}})
Recall the definitions of the sub-epochs corresponding to the pairs $(t_j, \epsilon_j)'s$ that satisfy the two conditions in Theorem~\ref{thm:IntermediateEpochProbability}. Following the same procedure as in the proof of \cite[Theorem 1.1]{Balsubramani2015}, notice that $J = \log_2{\big(1/(2\epsilon_0)\big)}$ (since $\epsilon_J = 2\epsilon_{J-1}=\dots=2^J \epsilon_0\Rightarrow 2^J = \epsilon_J/\epsilon_0=1/2\epsilon_0$) and therefore Condition~\ref{Condition:C2} implies
\begin{align}\label{eqn:t_J}
t_J + L + 1 = \big(L + 1\big)\exp{\Big(\frac{5 J}{c_0}\Big)}=\big(L + 1\big)\Big(\frac{1}{2\epsilon_0}\Big)^{5/(c_0 \ln{2})}=\big(L + 1\big)\Big(\frac{4 e d}{\delta^2}\Big)^{5/(c_0 \ln{2})}.
\end{align}
Defining $a_1 := c_0^2 \lambda_1^2/2(\lambda_1-\lambda_2)^2$, $a_2 := c_0 /2$, $b := 4c^2 \sigma_N^2$, and using Lemma~\ref{le:FinalRate} for $t>t_J$, we have
$$\E_t\{\Psi_t\}\leq \Big(1 +\frac{a_1}{(t+L)^2}- \frac{a_2}{t+L}\Big)\E_{t-1}\{\Psi_{t-1}\}+\frac{b}{(t+L)^2}.$$
Now using Proposition~\ref{prop:recursion} from Appendix~\ref{sec:App_FinalEpoch} with $c_0>2$, we get
\begin{align*}
\E_t \{\Psi_t\} &\leq \Big(\frac{t_J + L + 1}{t + L + 1}\Big)^{\frac{c_0}{2}}\exp\Big(\frac{a_1}{t_J+L+1}\Big)\E_{t_J}\{\Psi_{t_J}\} \\
&\qquad+ \frac{b}{a_2 - 1}\Big(1 + \frac{1}{t_J + L + 1}\Big)^{2}\exp\Big(\frac{a_1}{t_J + L +1}\Big)\frac{1}{t + L + 1}\\
&\stackrel{(a)}\leq \frac{1}{2}\Big(\frac{L + 1}{t + L + 1}\Big)^{\frac{c_0}{2}}\Big(\frac{4 e d}{\delta^2}\Big)^{\frac{5 a_2}{(c_0 \ln{2})}}\exp\Big(\frac{a_1}{t_J + L +1}\Big) \\
&\qquad+ \frac{b}{a_2 - 1}\exp{\Big(\frac{2}{t_J + L + 1}\Big)}\exp\Big(\frac{a_1}{t_J + L + 1}\Big)\frac{1}{t + L + 1}\\
&= \frac{1}{2}\Big(\frac{L + 1}{t + L + 1}\Big)^{\frac{c_0}{2}}\Big(\frac{4 e d}{\delta^2}\Big)^{\frac{5}{(2 \ln{2})}}\exp\Big(\frac{a_1}{(L + 1)(4ed/\delta^2)^{(5/2 \ln{2})}}\Big) \\
&\qquad+ \frac{8 c^2 \sigma_N^2}{c_0 - 2}\exp{\Big(\frac{2 + a_1}{(L + 1)(4ed/\delta^2)^{(5/2 \ln{2})}}\Big)}\frac{1}{(t + L + 1)}.
\end{align*}
Here, the inequality in $(a)$ is due to~\eqref{eqn:t_J} and we have also used the fact that $(1+x)^a\leq \exp{(ax)}$ for $x<1$. In addition, since $(4ed/\delta^2)^{(5/2 \ln{2})} \geq 1$, we get
\begin{align}\label{eqn:Full}
\E_t \{\Psi_t\} &\leq \frac{1}{2}\Big(\frac{L + 1}{t + L + 1}\Big)^{\frac{c_0}{2}}\Big(\frac{4 e d}{\delta^2}\Big)^{\frac{5}{(2 \ln{2})}}\exp{\Big(\frac{a_1}{L + 1}\Big)} + \frac{8 c^2 \sigma_N^2}{c_0 -2}\exp{\Big(\frac{a_1 + 2}{L + 1}\Big)}\frac{1}{(t + 1)}\nonumber\\
&\leq \frac{1}{2}\Big(\frac{L + 1}{t + L + 1}\Big)^{\frac{c_0}{2}}\Big(\frac{4 e d}{\delta^2}\Big)^{\frac{5}{(2 \ln{2})}}e^{a_1/L} + \frac{8 c^2 \sigma_N^2 e^{(a_1 + 2)/L}}{c_0 -2}\frac{1}{(t + L + 1)}\nonumber\\
&= C_1 \Big(\frac{L + 1}{t + L + 1}\Big)^{\frac{c_0}{2}} + C_2\Big(\frac{\sigma_N^2}{t + L + 1}\Big).
\end{align}
This completes the proof of the theorem.
\end{proof}

\section{Numerical Results}\label{sec:NumericalResults}
In this section, we utilize numerical experiments to validate the theoretical findings of this work in terms of the ability of implicit/explicit mini-batched variants of the original Krasulina's method \citep{krasulina1969method} to estimate the top eigenvector of a covariance matrix from (fast) streaming data. Instead of repeating the same set of experiments for the original Krasulina's method, \DK, and \DMK, we present our results that are parameterized by the network-wide mini-batch size $B \in \{1\} \bigcup \{bN: b \in \Z_+\}$ that appears in \DMK. This is because $B = 1$ trivially corresponds to the original Krasulina's iterations, while $B = N$ corresponds to iterations that characterize \DK.

Our goals for the numerical experiments are threefold: ($i$) showing the impact of (implicit/explicit) mini-batching on the convergence rate of \DMK, ($ii$) establishing robustness of {\DMK} against the loss of $\mu > 0$ samples per iteration for the case when $N < \tfrac{R_s}{R_p} + \tfrac{R_s}{b R_c}$, and ($iii$) experimental validation for scaling of convergence rate in terms of problem parameters as predicted by our theoretical findings, namely, eigengap ($\lambda_1-\lambda_2$), dimensionality ($d$), and upper bound on input samples ($\|\bx_{t'}\|_2\leq r$). In the following, we report results of experiments on both synthetic and real-world data to highlight these points.

\subsection{Experiments on Synthetic Data}\label{subsec:SyntheticDataNoLatency}
In the following experiments we generate $T=10^6$ samples from some probability distribution (specified for each experiment later) and for each experiment we perform $200$ Monte-Carlo trials. In all the experiments in the following we use step size of the form $\gamma_t=c/t$. We performed experiments with multiple values of $c$ and here we are reporting the results for the value of $c$ which achieves the best convergence rate. Further details about each experiment are provided in the following sections.

\subsubsection{Impact of mini-batch size on the performance of \DMK}

\begin{figure}[t]
	\centering
	\subfigure[Impact of the mini-batch size on the convergence rate of \DMK~for the resourceful regime. Note that the $B=1$ plot is effectively Krasulina's method.]{
		\includegraphics[width=0.45\columnwidth]{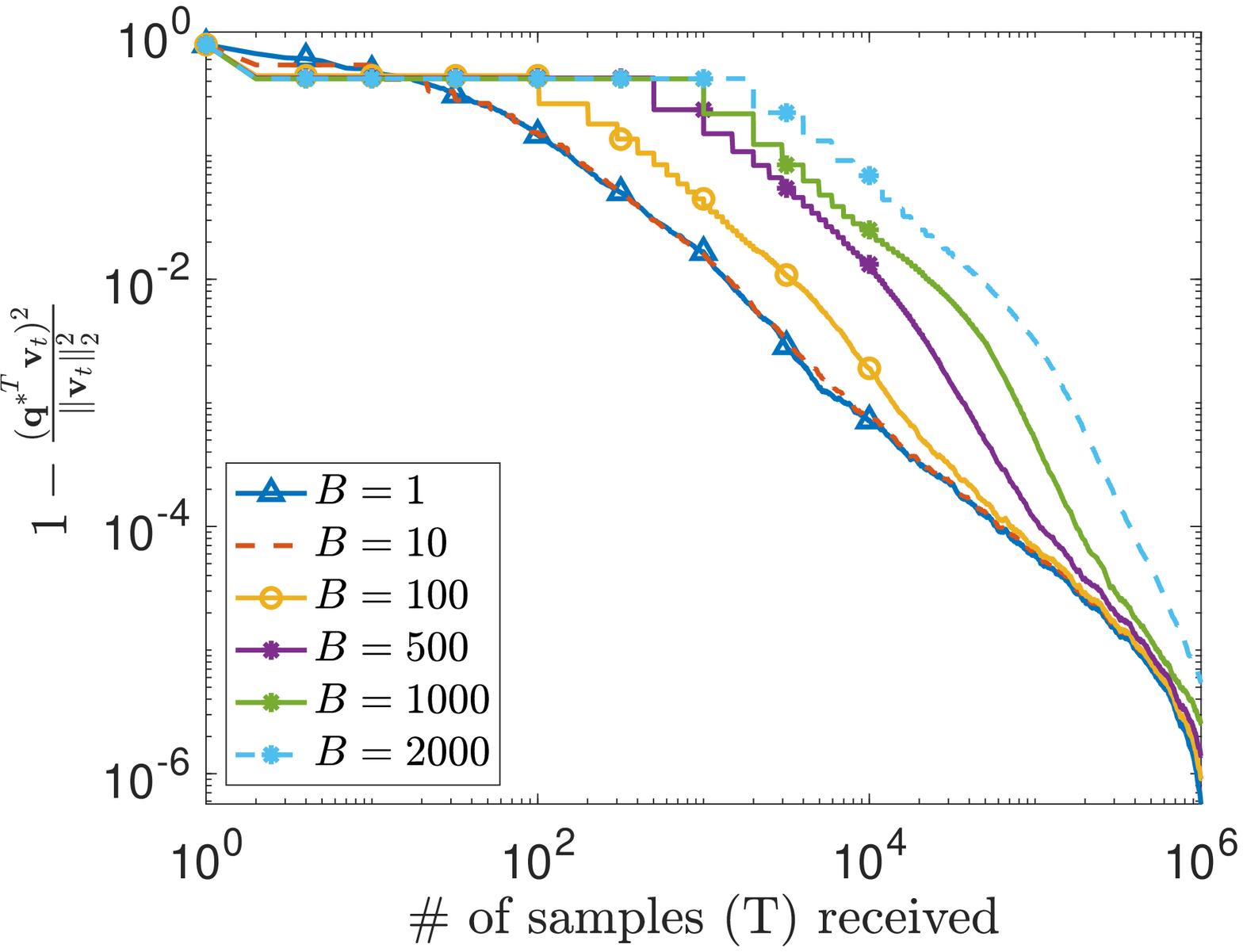}
		\label{fig:SyntheticPCA_NoLatency}}
	\qquad
	\subfigure[Performance of \DMK~in a resource-constrained regime (i.e., $N < \tfrac{R_s}{R_p} + \tfrac{R_s}{b R_c}$), which causes loss of $\mu$ samples per iteration; here, $(N,B)=(10,100)$.]{
		\includegraphics[width=0.45\columnwidth]{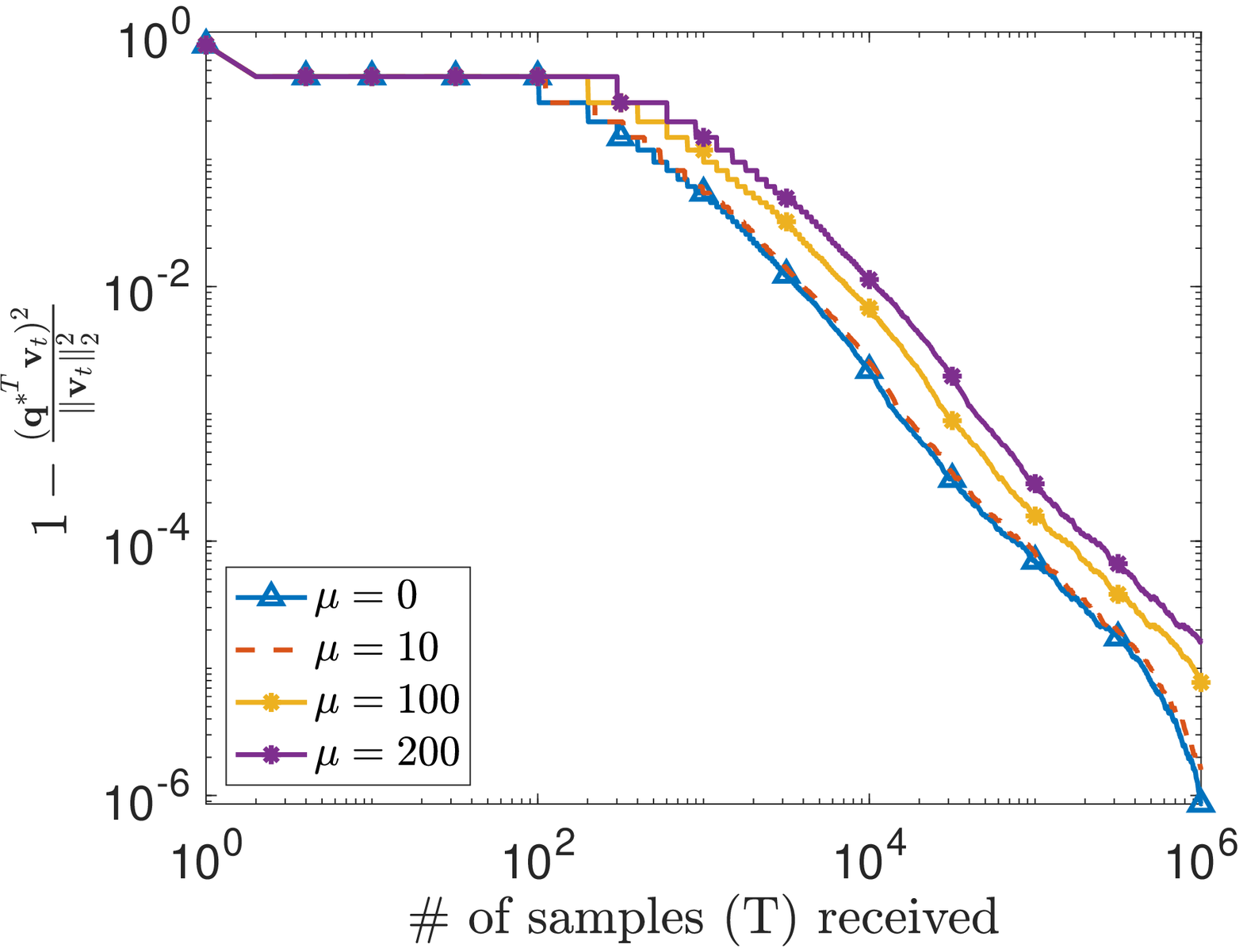}
		\label{fig:SyntheticPCA_Latency}}
	\caption{Convergence behavior of \DMK~for the case of synthetic data under two scenarios: (a) No data loss ($\mu = 0$) and (b) loss of $\mu > 0$ samples per algorithmic iteration.}
\end{figure}

For a covariance matrix $\bSigma\in\R^{5\times 5}$ with $\lambda_1=1$ and eigengap $\lambda_1 - \lambda_2=0.2$, we generate $T=10^6$ samples from $\cN(\bzero, \bSigma)$ distribution. The first set of experiments here deals with the resourceful regime, i.e., $N \geq \tfrac{R_s}{R_p} + \tfrac{R_s}{b R_c}$, with mini-batches of sizes $B \in \{1, 10, 100, 500, 1000, 2000\}$. Note that these values of $B$ can be factored into any positive integers $b$ and $N$ as long as the condition $N \geq \tfrac{R_s}{R_p} + \tfrac{R_s}{b R_c}$ that is governed by the application scenario and the physical system is satisfied. It is, therefore, unnecessary to specify $b$ and $N$ for these experiments, whose results are shown in Figure~\ref{fig:SyntheticPCA_NoLatency}. These results are obtained for step-size parameter $c \in \{70,80,80,90,110,100\}$, which are the values of $c$ resulting in the best convergence rate. As predicted by Corollary~\ref{cor:BatchNoLatency}, we can see that after $T/B$ iterations of \DMK, the error $\Psi_{T/B}$ is on the order of $O(1/T)$ for $B \in \{1, 10, 100, 500, 1000\}$, while for $B=2000$, the error $\Psi_{T/B}$ is not optimal anymore.

Next, we demonstrate the performance of \DMK~for resource constrained settings, i.e., $N < \tfrac{R_s}{R_p} + \tfrac{R_s}{b R_c}$, which causes the algorithm to discard $\mu := \left(\tfrac{b R_s}{R_p} + \tfrac{R_s}{R_c}\right) - B$ samples per iteration. Using the same data generation setup as before, we run \DMK~for a network of 10 nodes ($N=10$) with network-wide mini-batch of size $B=100$ (i.e., $b=10$). We consider different mismatch factors between streaming, processing, and communication rates in this experiment, which result in the number of samples being discarded as $\mu \in \{0, 10, 100, 200\}$. The results are plotted in Figure~\ref{fig:SyntheticPCA_Latency}, which shows that the error $\Psi_{T/(B+\mu)}$ for $\mu=10$ is comparable to that for $\mu=0$, but the error for $\mu=200$ is an order of magnitude worse than the nominal error.

\subsubsection{Impact of the eigengap on the performance of \DMK}\label{subsubsec:eigengap}
For this set of experiments, we again generate data in $\R^5$ from a normal distribution $\cN(\bzero, \bSigma)$, where the covariance matrix $\bSigma$ has the largest eigenvalue $\lambda_1=1$. We then vary the remaining eigenvalues to ensure an eigengap that takes values from the set $\{0.1, 0.2, 0.3, 0.4, 0.5\}$. The corresponding values of $c$ that give the best convergence rate for each unique eigengap satisfy $c \in \{180,110,90,70,60\}$. The final results for these experiments are plotted in Figure~\ref{fig:Eigengap} for the case of $B=1000$ and $\mu = 0$. These results establish that the final gap in error after observing $T=10^6$ data samples is indeed on the order of $O(1/(\lambda_1-\lambda_2)^2)$, as suggested by the theoretical analysis.

\subsubsection{Impact of dimensionality on the performance of \DMK}
For this set of experiments, we generate data in $\R^d$ from a normal distribution $\cN(\bzero, \bSigma)$ whose dimensionality is varied such that $d \in \{5, 10, 15, 20\}$. In addition, we fix the largest eigenvalue of $\bSigma$ to be $\lambda_1=1$ and its eigengap to be $0.2$. The values of $c$ corresponding to each unique value of $d$ that provide the best convergence rate in these experiments satisfy $\{110, 110, 100, 100\}$; contrary to our theoretical analysis, this seems to suggest that the optimal step-size sequence does not have a strong dependence on $d$, at least for small values of $d$. We also plot the potential function for each $d$ as a function of the number of received samples in Figure~\ref{fig:Dimensions} for the case of $B=1000$ and $\mu = 0$. Once again, we observe little dependence of the performance of \DMK~on $d$. Both these observations suggest that our theoretical analysis is not tight in terms of its dependence on dimensionality $d$ of the streaming data.

\begin{figure}[t]
	\centering
	\subfigure[]{
		\includegraphics[width=0.45\columnwidth]{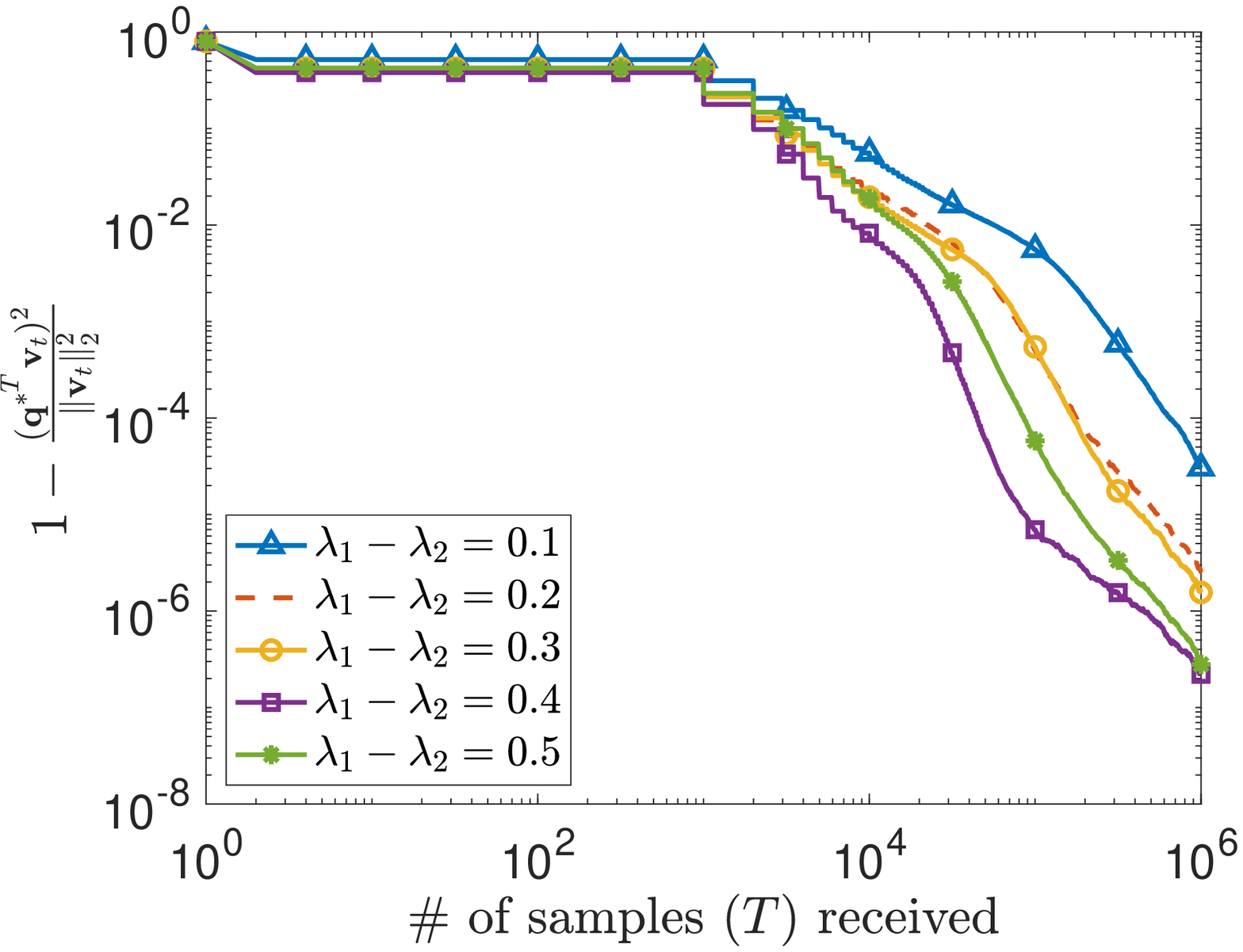}
		\label{fig:Eigengap}}
	\qquad
	\subfigure[]{
		\includegraphics[width=0.45\columnwidth]{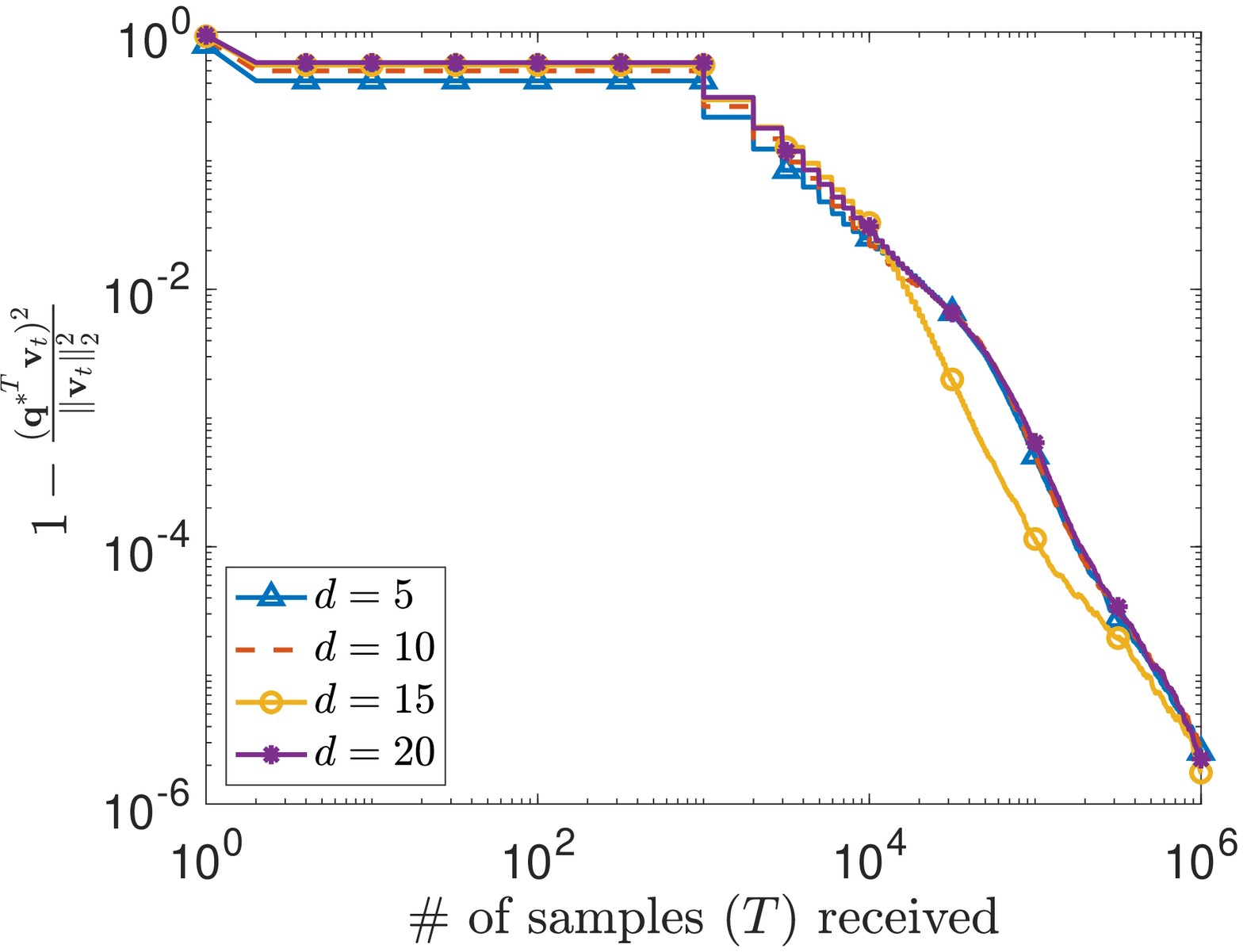}
		\label{fig:Dimensions}}
	\caption{Understanding the impact of (a) eigengap $(\lambda_1-\lambda_2)$ and (b) dimensionality $d$ on the convergence behavior of \DMK, corresponding to $B=1000$ and $\mu = 0$.}
\end{figure}

\begin{figure}[t]
	\centering
	\includegraphics[width=0.5\columnwidth]{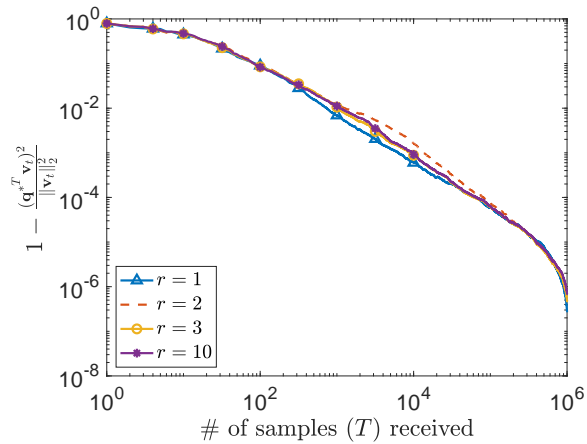}
	\caption{Performance of \DMK~for varying upper bound on the norm of the streaming data.}
	\label{fig:Upperbound}
\end{figure}

\subsubsection{Impact of upper bound on the performance of \DMK}
In order to understand the impact of the upper bound $\|\bx_{t'}\|_2\leq r$ on the convergence behavior of \DMK, we generate $\bx_{t'} \in \R^5$ as $\bx_{t'} = \bC \bu_{t'}$ with $\bu_{t'} \in \R^5$ having independent entries drawn from uniform distribution $\cU(-a, a)$ and $\bC$ chosen to ensure an eigengap of $0.2$ for the covariance matrix. As we vary the value of $a$ within the set $\{1,2,3,10\}$, we generate four different datasets of $T = 10^6$ samples for which the resulting $r \in \{1.45, 2.9, 4.5, 14.5\}$. The values of $c$ that provide best convergence for these values of $r$ satisfy $c \in \{8, 2, 1, 0.08\}$. The final set of results are displayed in Figure~\ref{fig:Upperbound} for $B=1$ and $\mu = 0$. It can be seen from this figure that changing $r$ does not affect the convergence behavior of \DMK. This behavior can be explained by noticing that the parameter $r$ appears in our convergence results in terms of a lower bound on $L$ (cf.~\eqref{eqn:LowerboundL}) and within the non-dominant term in the error bound. The dependence of $L$ on the parameter $r$ is already being reflected here in our choice of the step-size parameter $c$ that results in the best convergence result. In addition, we hypothesize that the non-dominant error term in our experiments, compared to the dominant one, is significantly small that it masks the dependence of the final error on $r$.

\begin{figure}
	\centering
		\subfigure[MNIST Data ($\mu=0$): Impact of network-wide mini-batch size $B$ on the convergence behavior of \DMK~for the resourceful regime.]{
	\includegraphics[width=0.45\columnwidth]{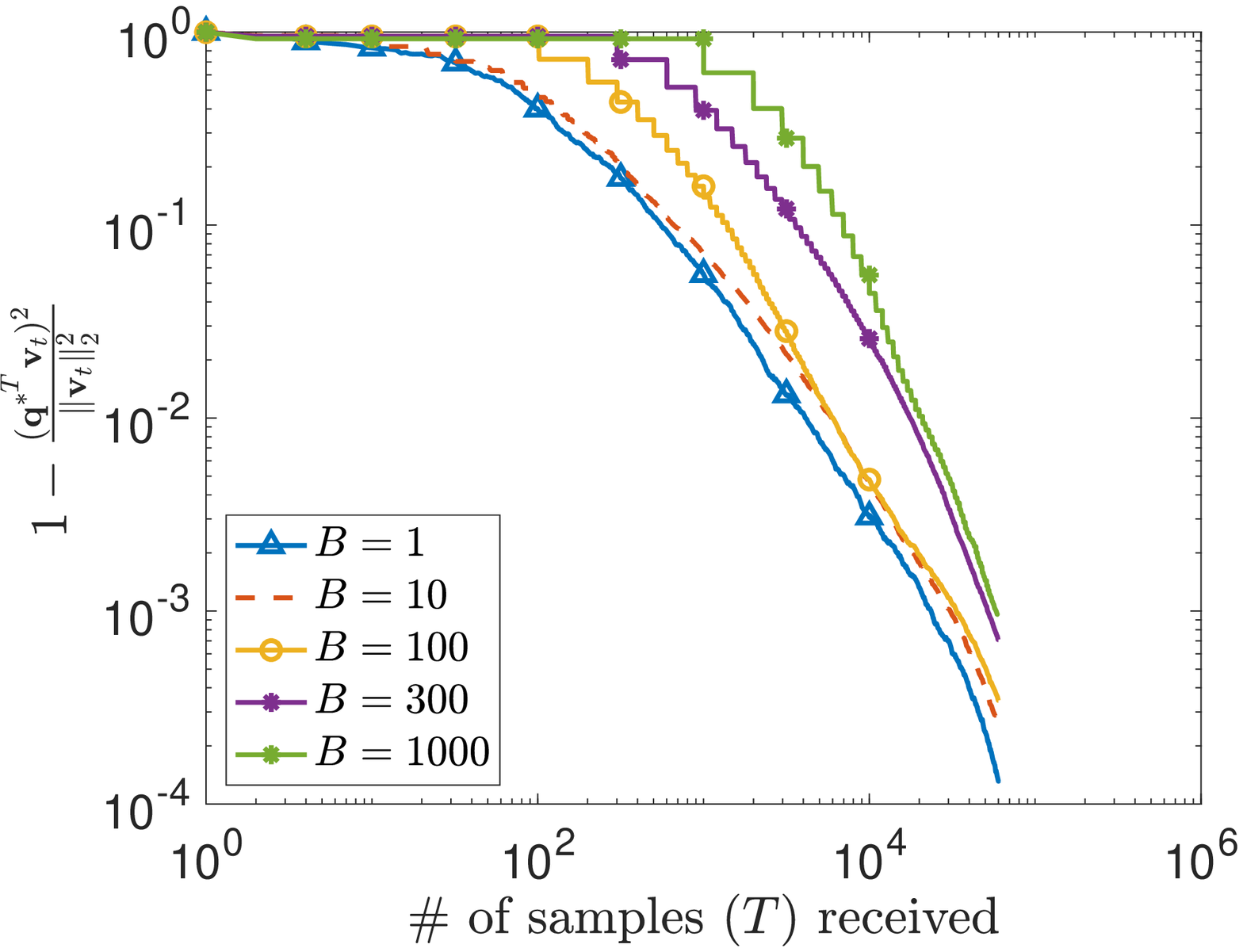}
        \label{fig:MNIST_NoLatency}}
		\qquad
		\centering
	    	\subfigure[MNIST Data ($N=10$; $B=100$): Convergence behavior of \DMK~in a resource-constrained regime, which causes loss of $\mu$ samples per iteration.]{
	    	\includegraphics[width=0.45\columnwidth]{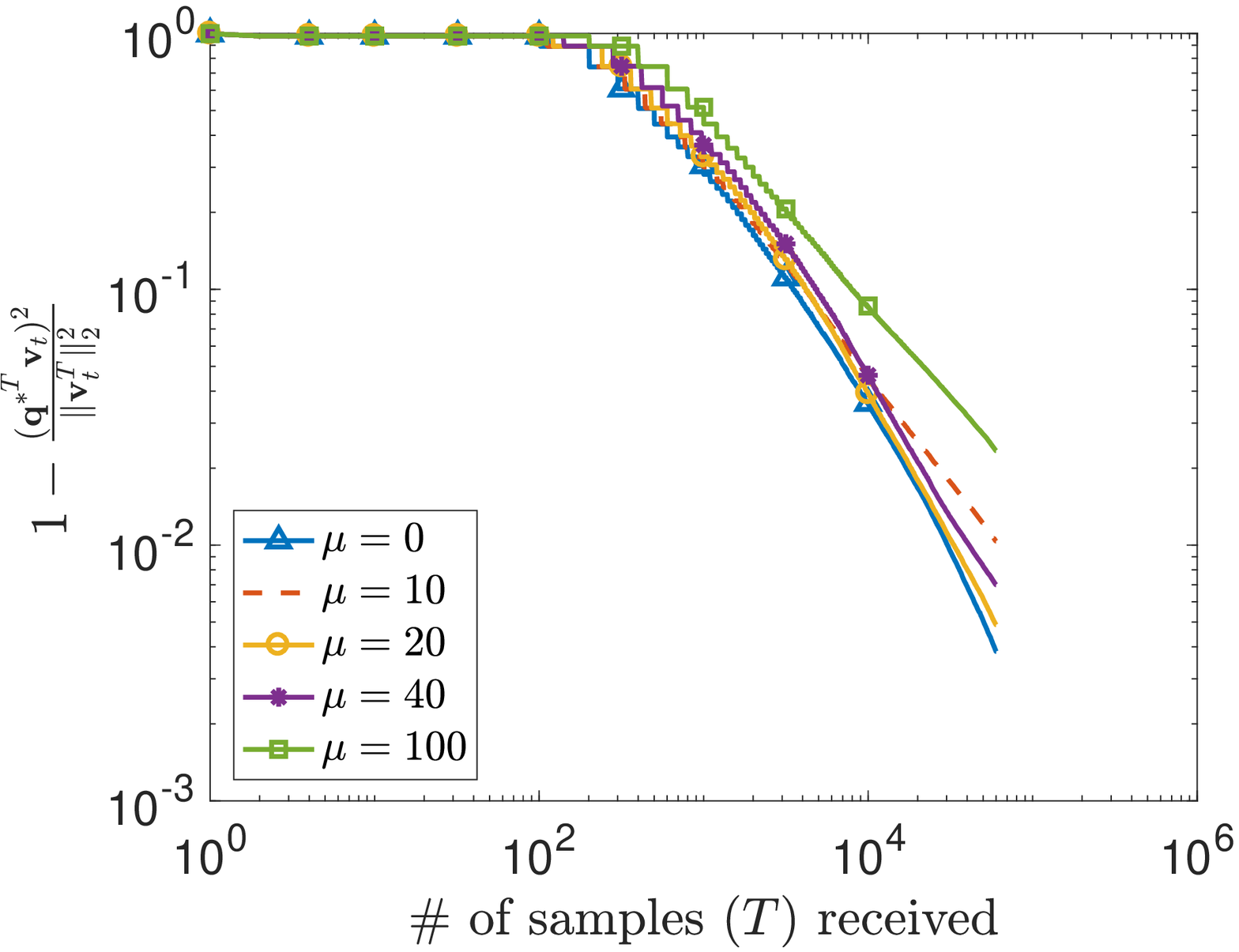}
            \label{fig:MNIST_Latency}}
		\caption{Performance of \DMK~for the MNIST dataset under two scenarios: (a) No data loss ($\mu = 0$) and (b) loss of $\mu > 0$ samples per algorithmic iteration.}
\end{figure}

\subsection{Experiments on Real-world Datasets}\label{subsec:RealData}
In this section, we evaluate the performance of \DMK~on two real-world datasets, namely, the MNIST dataset~\citep{lecun1998mnist} and the Higgs dataset~\citep{baldi2014searching}. The MNIST dataset corresponds to $d=784$ and has a total of $T = 6\times 10^4$ samples. Our first set of experiments for this dataset uses the step size $\gamma=c/t$ with $c \in \{0.6, 0.9, 1.1, 1.5, 1.6\}$ for network-wide mini-batch sizes $B \in \{1, 10, 100, 300, 1000\}$ in the resourceful regime ($\mu = 0$). The results, which are averaged over 200 random initializations and random shuffling of data, are given in Figure~\ref{fig:MNIST_NoLatency}. It can be seen from this figure that the final error relatively stays the same as $B$ increases from $1$ to $100$, but it starts getting affected significantly as the network-wide mini-batch size is further increased to $B=300$ and $B=1000$. Our second set of experiments for the MNIST dataset corresponds to the resource-constrained regime with $(N,B)=(10,100)$ and step-size parameter $c \in \{0.6, 0.9, 1.1, 1.5, 1.6\}$ for the number of discarded samples $\mu \in \{0, 10, 20, 40, 100\}$. The results, averaged over 200 trials and given in Figure~\ref{fig:MNIST_Latency}, show that the system can tolerate loss of some data samples per iteration without significant increase in the final error; the increase in error, however, becomes noticeable as $\mu$ approaches $B$. Both these observations are in line with the insights of our theoretical analysis.

We now turn our attention to the Higgs dataset, which is $d=28$ dimensional and comprises $1.1\times 10^7$ samples. Our results for this dataset, averaged over 200 trials and using $c = 0.07$, for the resourceful and resource-constrained settings are given in Figure~\ref{fig:Higgs_NoLatency} and Figure~\ref{fig:Higgs_Latency}, respectively. In the former setting, corresponding to $B \in \{1, 10^2, 10^3, 10^4, 2\times 10^4\}$, we once again see that the error relatively stays the same for values of $B$ that are significantly smaller than $T$; in particular, since $T$ for the Higgs dataset is larger than for the MNIST dataset, it can accommodate a larger value of $B$ without significant loss in performance. In the latter resource-constrained setting, corresponding to $N = 10$, $B = 1000$ and $\mu \in \{0,10,100,1000,2000\}$, we similarly observe that small (relative to $B$) values of $\mu$ do not impact the performance of \DMK~in a significant manner. Once again, these results corroborate our research findings.

\begin{figure}
	\centering
    \subfigure[Higgs Data ($\mu=0$): Impact of network-wide mini-batch size $B$ on the convergence behavior of \DMK~for the resourceful regime.]{
    	\includegraphics[width=0.45\columnwidth]{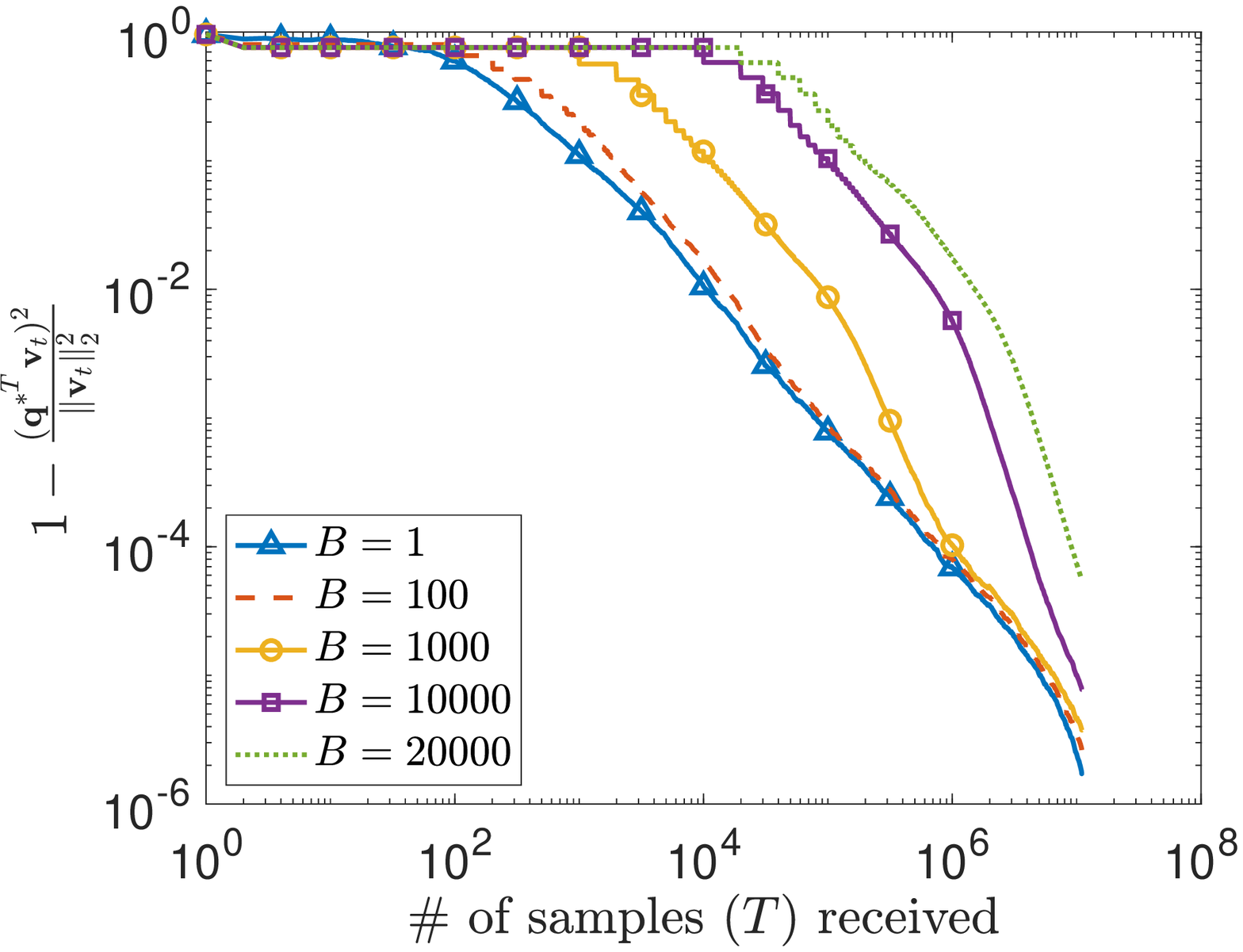}
            \label{fig:Higgs_NoLatency}}
    \qquad
    \centering
    \subfigure[Higgs Data ($N=10$; $B=1000$): Convergence behavior of \DMK~in a resource-constrained regime, which causes loss of $\mu$ samples per iteration.]{
    	\includegraphics[width=0.45\columnwidth]{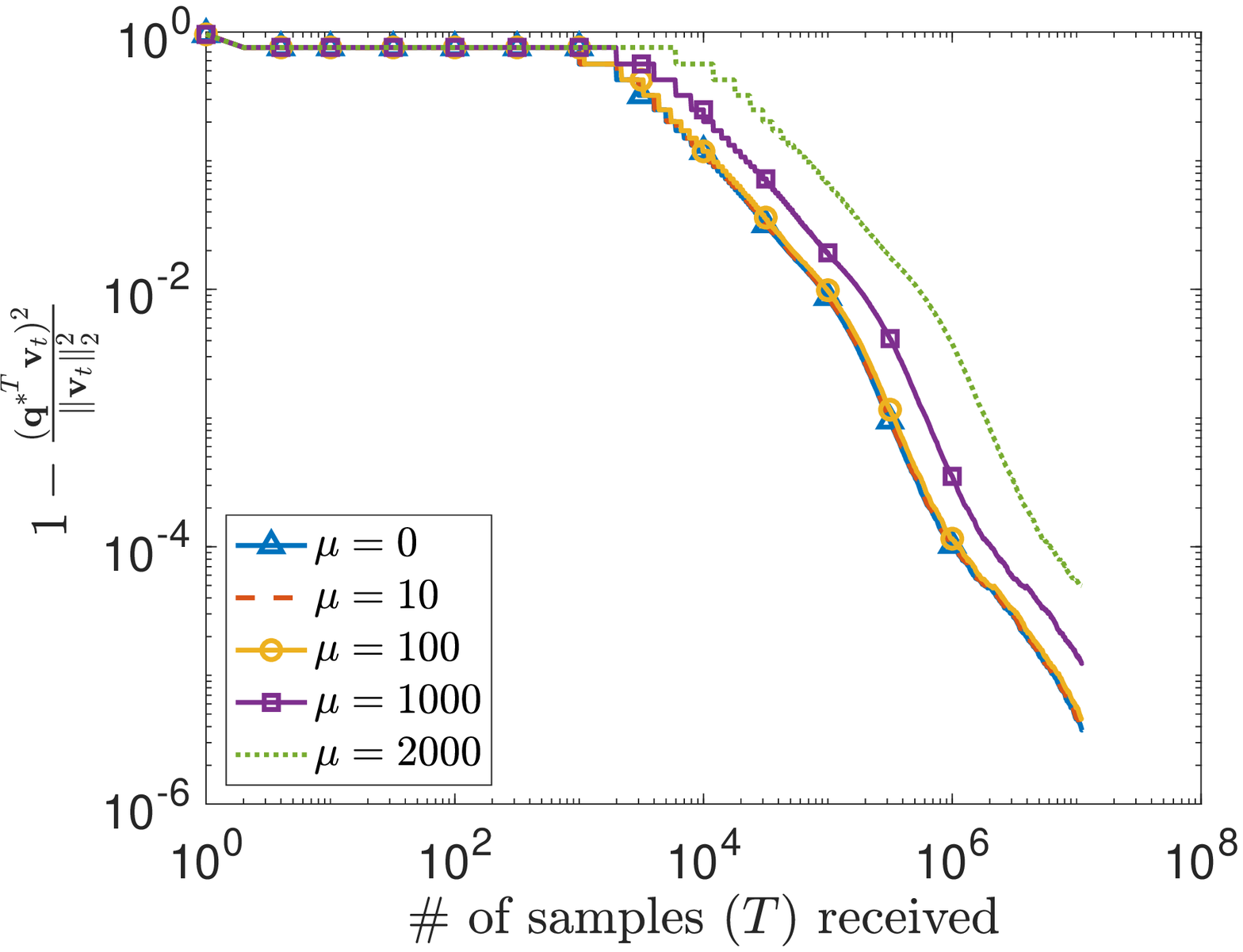}
            \label{fig:Higgs_Latency}}
	\caption{Performance of \DMK~for the Higgs dataset under two scenarios: (a) No data loss ($\mu = 0$) and (b) loss of $\mu > 0$ samples per algorithmic iteration.}
\end{figure} 
\section{Conclusion}\label{sec:Conclusion}
In this paper, we studied the problem of estimating the principal eigenvector of a covariance matrix from independent and identically distributed data samples. Our particular focus in here was developing and analyzing two variants, termed \DK~and \DMK, of a classical stochastic algorithm that can estimate the top eigenvector in a near-optimal fashion from fast streaming data that overwhelms the processing capabilities of a single processor. Unlike the classical algorithm that must discard data samples in high-rate streaming settings, and thus sacrifice the convergence rate, the proposed algorithms manage the high-rate streaming data by trading off processing capabilities with computational resources and communications infrastructure. Specifically, both \DK~and \DMK~virtually slow down the rate of streaming data by spreading the processing of data samples across of a network of processing nodes. In addition, \DMK~can overcome slower communication links and/or lack of sufficient number of processing nodes through a network-wide mini-batching strategy, coupled with discarding of a small number of data samples per iteration.

Our theoretical analysis, which fundamentally required a characterization of the error incurred by the proposed algorithms as a function of the variance of the sample covariance matrix, established the conditions under which near-optimal convergence rate is achievable in the fast streaming setting, even when some data samples need to be discarded due to lack of sufficient computational and/or communication resources. We also carried out numerical experiments on both synthetic and real-world data to validate our theoretical findings.

In terms of future work, extension of our algorithmic and analytical framework for estimation of the principal subspace comprising multiple eigenvectors remains an open problem. In addition, tightening our theoretical analysis to better elucidate the role of dimensionality of data in the performance of the proposed algorithmic framework is an interesting problem.

\begin{appendices}
\section{Proofs of Lemmas for the Initial Epoch}\label{secApp:App_InitialPhase}

\subsection{Proof of Lemma~\ref{le:ErrorReccursion}}\label{secApp:App_InitialPhase_1}
In order to prove Lemma~\ref{le:ErrorReccursion}, we first need the following result.
	\begin{lemma}\label{le:xi_t_bound}
		The second moment of the update vector $\bxi_t$ in \DK~is upper bounded as
		$$\E\Bigg\{\frac{\|\bxi_t\|_2^2}{\|\bv_{t-1}\|_2^2}\Bigg\}\leq \frac{\mathbb{E}\left\{\|\bxi_t - \mathbb{E}\bxi_t\|_2^2\right\}}{\|\bv_{t-1}\|_2^2} + 2 \lambda_1^2\Psi_{t-1}.$$
	\end{lemma}
	\begin{proof}
		We start by writing $\E\left\{\|\bxi_t - \E\{\bxi_t\}\|_2^2\right\}$ in terms of $\E\left\{\|\bxi_t\|_2^2\right\}$ as follows:
		\begin{align*}
		\E\left\{\|\bxi_t - \E\{\bxi_t\}\|_2^2\right\} &= \E\Bigg\{\bxi_t^{\tT}\bxi_t + (\E\{\bxi_t\})^{\tT}\E\{\bxi_t\} - \bxi_t^{\tT}\E\{\bxi_t\} - (\E\{\bxi_t\})^{\tT}\bxi_t\Bigg\}\\
		&=\E\{\|\bxi_t\|_2^2\} - \E\{\bxi_t^{\tT}\} \E\{\bxi_t\}.
		\end{align*}
		Now defining $C_t:=\E\{\bxi_t^{\tT}\}\E\{\bxi_t\}$ and rearranging the above equation, we get
		$$\mathbb{E}\{\|\bxi_t\|_2^2\} = \mathbb{E}\{\|\bxi_t - \mathbb{E}\{\bxi_t\}\|_2^2\} + C_t.$$
		Next, substituting value of $\bxi_t$ from~\eqref{eqn:MiniBatchOjaRule} we get
		\begin{align}\label{eqn:XiNormDeterministic}
		\frac{C_t}{\|\bv_{t-1}\|_2^2} = \frac{\E\{\bxi_t^{\tT}\}\E\{\bxi_t\}}{\|\bv_{t-1}\|_2^2} &= \frac{1}{\|\bv_{t-1}\|_2^2}\Bigg(\bSigma \bv_{t-1}-\frac{\bv_{t-1}^{\tT}\bSigma \bv_{t-1}\bv_{t-1}}{\bv_{t-1}^{\tT}\bv_{t-1}}\Bigg)^{\tT}\Bigg(\bSigma \bv_{t-1}-\frac{\bv_{t-1}^{\tT}\bSigma \bv_{t-1}\bv_{t-1}}{\bv_{t-1}^{\tT}\bv_{t-1}}\Bigg)\nonumber\\
		&= \frac{\bv_{t-1}^{\tT}\bSigma^2 \bv_{t-1}}{\|\bv_{t-1}\|_2^2}-\Bigg(\frac{\bv_{t-1}^{\tT}\bSigma \bv_{t-1}}{\|\bv_{t-1}\|_2^2}\Bigg)^2.
		\end{align}
		Since $\bSigma$ is a positive semi-definite matrix, we can write its eigenvalue decomposition as
		$\bSigma = \sum_{i=1}^{d}\lambda_i \bq_i \bq_i^{\tT},$
		where $\lambda_1>\lambda_2\geq\dots\geq \lambda_d \geq 0$ and $\bq_1 (\equiv \bq^*),\bq_2,\dots,\bq_d$ are the eigenvalues and corresponding eigenvectors of $\bSigma$, respectively. It follows that
		\begin{align*}
		\frac{C_t}{\|\bv_{t-1}\|_2^2}
		&= \sum_{i=1}^{d}{\lambda_i^2 \frac{(\bv_{t-1}^{\tT}\bq_i)^2}{\|\bv_{t-1}\|_2^2}}-\Bigg(\sum_{i=1}^{d}{\lambda_i \frac{(\bv_{t-1}^{\tT}\bq_i)^2}{\|\bv_{t-1}\|_2^2}}\Bigg)^2\nonumber\\
		&= \lambda_1^2 \frac{(\bv_{t-1}^{\tT}\bq^*)^2}{\|\bv_{t-1}\|_2^2} + \sum_{i=2}^{d}{\lambda_i^2 \frac{(\bv_{t-1}^{\tT}\bq_i)^2}{\|\bv_{t-1}\|_2^2}}-\Bigg(\lambda_1 \frac{(\bv_{t-1}^{\tT}\bq^*)^2}{\|\bv_{t-1}\|_2^2} + \sum_{i=2}^{d}{\lambda_i \frac{(\bv_{t-1}^{\tT}\bq_i)^2}{\|\bv_{t-1}\|_2^2}}\Bigg)^2\nonumber\\
		&\leq \lambda_1^2 \frac{(\bv_{t-1}^{\tT}\bq^*)^2}{\|\bv_{t-1}\|_2^2} + \lambda_2^2 \sum_{i=2}^{d}{ \frac{(\bv_{t-1}^{\tT}\bq_i)^2}{\|\bv_{t-1}\|_2^2}}-\lambda_1^2 \frac{(\bv_{t-1}^{\tT}\bq^*)^4}{\|\bv_{t-1}\|_2^4}\nonumber\\
		&= \lambda_1^2 \frac{(\bv_{t-1}^{\tT}\bq^*)^2}{\|\bv_{t-1}\|_2^2}\Bigg(1 - \frac{(\bv_{t-1}^{\tT}\bq^*)^2}{\|\bv_{t-1}\|_2^2}\Bigg) + \lambda_2^2 \Bigg(1 - \frac{(\bv_{t-1}^{\tT}\bq^*)^2}{\|\bv_{t-1}\|_2^2}\Bigg).
		\end{align*}
		Finally, we get from definition of $\Psi_{t-1}$ that
		$$\frac{C_t}{\|\bv_{t-1}\|_2^2}\leq \Psi_{t-1}\big((1 - \Psi_{t-1})\lambda_1^2+\lambda_2^2\big)\leq \Psi_{t-1}\big(\lambda_1^2+\lambda_2^2\big)\leq 2\lambda_1^2\Psi_{t-1}.$$
    This completes the proof of the lemma.
	\end{proof}

    Using Lemma~\ref{le:xi_t_bound}, we can now prove Lemma~\ref{le:ErrorReccursion} in the following.
\begin{proof}[Proof of Lemma~\ref{le:ErrorReccursion}] From~\eqref{eqn:Error}, we have
	$\Psi_t = \frac{\|\bv_t\|_2^2-(\bv_t^{\tT}\bq^*)^2}{\|\bv_t\|_2^2}.$
	Substituting $\bv_t$ from~\eqref{eqn:MiniBatchOjaRule}, we get
	\begin{align}\label{eqn:PsiRecurrence}
	\Psi_t &= \frac{\|\bv_{t-1}+\gamma_t\bxi_t\|_2^2-((\bv_{t-1}+\gamma_t \bxi_t)^{\tT}\bq^*)^2}{\|\bv_{t}\|_2^2} \stackrel{(a)}= \frac{\|\bv_{t-1}\|_2^2+\gamma_t^2\|\bxi_t\|_2^2-((\bv_{t-1}+\gamma_t \bxi_t)^{\tT}\bq^*)^2}{\|\bv_{t}\|_2^2}\nonumber\\
	&\stackrel{(b)}\leq \frac{\|\bv_{t-1}\|_2^2+\gamma_t^2\|\bxi_t\|_2^2-((\bv_{t-1}+\gamma_t \bxi_t)^{\tT}\bq^*)^2}{\|\bv_{t-1}\|_2^2} = 1+\gamma_t^2\frac{\|\bxi_t\|_2^2}{\|\bv_{t-1}\|_2^2}-\frac{((\bv_{t-1}+\gamma_t \bxi_t)^{\tT}\bq^*)^2}{\|\bv_{t-1}\|_2^2}\nonumber\\
	&= 1+\gamma_t^2\frac{\|\bxi_t\|_2^2}{\|\bv_{t-1}\|_2^2}- \frac{(\bv_{t-1}^{\tT}\bq^*)^2+\gamma_t^2( \bxi_t^{\tT}\bq^*)^2 + 2\gamma_t (\bv_{t-1}^{\tT}\bq^*)(\bxi_t^{\tT}\bq^*)}{\|\bv_{t-1}\|_2^2}\nonumber\\
	&= 1 - \frac{(\bv_{t-1}^{\tT}\bq^*)^2}{\|\bv_{t-1}\|_2^2} +\gamma_t^2\frac{\|\bxi_t\|_2^2-( \bxi_t^{\tT}\bq^*)^2}{\|\bv_{t-1}\|_2^2} - 2\gamma_t \frac{(\bv_{t-1}^{\tT}\bq^*)(\bxi_t^{\tT}\bq^*)}{\|\bv_{t-1}\|_2^2}\nonumber\\
	&= \Psi_{t-1} +\gamma_t^2\frac{\|\bxi_t\|_2^2}{\|\bv_{t-1}\|_2^2} - 2\gamma_t \frac{(\bv_{t-1}^{\tT}\bq^*)(\bxi_t^{\tT}\bq^*)}{\|\bv_{t-1}\|_2^2}.
	\end{align}
	Here ($a$) and ($b$) are due to~\cite[Lemma A.1]{Balsubramani2015}, where ($a$) is true because $\bv_{t-1}$ is perpendicular to $\bxi_t$ and ($b$) is true because $\|\bv_{t-1}\|_2\leq \|\bv_t\|_2$. The second term in the above inequality can be bounded as
	\begin{align}\label{eqn:XiNorm}
	\frac{\|\bxi_t\|_2^2}{\|\bv_{t-1}\|_2^2}&=\frac{\|\bxi_t - \E\{\bxi_t\}\|_2^2 + \E\{\bxi_t^{\tT}\}\E\{\bxi_t\}}{\|\bv_t\|_2^2} \stackrel{(c)}{\leq} \frac{\mathbb{E}\left\{\|\bxi_t - \mathbb{E}\{\bxi_t\}\|_2^2\right\}}{\|\bv_{t-1}\|_2^2} +2\lambda_1^2 \Psi_{t-1}\nonumber\\
	&= \frac{1}{\|\bv_{t-1}\|_2^2}\E\Bigg\{\Bigg\|\frac{1}{N}\sum_{i=1}^{N}\bA_{i,t} \bv_{t-1}-\frac{1}{\|\bv_{t-1}\|_2^2}\Big(\bv_{t-1}^{\tT}\frac{1}{N}\sum_{i=1}^{N}\bA_{i,t} \bv_{t-1} \bv_{t-1}\Big) \nonumber\\
	&\qquad - \E\Big\{\frac{1}{N}\sum_{i=1}^{N}\bA_{i,t} \bv_{t-1}-\frac{1}{\|\bv_{t-1}\|_2^2}\Big(\bv_{t-1}^{\tT}\frac{1}{N}\sum_{i=1}^{N}\bA_{i,t} \bv_{t-1} \bv_{t-1}\Big)\Big\}\Bigg\|_2^2\Bigg\}\nonumber\\
    &= \frac{1}{\|\bv_{t-1}\|_2^2}\E\Bigg\{\Big\|\frac{1}{N}\sum_{i=1}^{N}\bA_{i,t} \bv_{t-1}-\frac{1}{\|\bv_{t-1}\|_2^2}\Big(\bv_{t-1}^{\tT}\frac{1}{N}\sum_{i=1}^{N}\bA_{i,t} \bv_{t-1} \bv_{t-1}\Big) \nonumber\\
    &\qquad - \bSigma \bv_{t-1}+\frac{1}{\|\bv_{t-1}\|_2^2}\Big(\bv_{t-1}^{\tT}\bSigma \bv_{t-1} \bv_{t-1}\Big)\Big\|_2^2\Bigg\}\nonumber\\
    &= \frac{1}{\|\bv_{t-1}\|_2^2}\E\Bigg\{\Big\|\Big(\frac{1}{N}\sum_{i=1}^{N}\bA_{i,t} - \bSigma\Big)\bv_{t-1} -\frac{1}{\|\bv_{t-1}\|_2^2}\bv_{t-1}^{\tT}\Big(\frac{1}{N}\sum_{i=1}^{N}\bA_{i,t}-\bSigma \Big)\bv_{t-1} \bv_{t-1}\Big\|_2^2\Bigg\} \nonumber\\
	&\leq 4\Bigg\|\frac{1}{N}\sum_{i=1}^{N}\bA_{i,t} - \bSigma\Bigg\|_2^2+2\lambda_1^2 \Psi_{t-1} \leq 4\Bigg\|\frac{1}{N}\sum_{i=1}^{N}\bA_{i,t} - \bSigma\Bigg\|_F^2+2\lambda_1^2 \Psi_{t-1},
	\end{align}
where ($c$) is due to Lemma~\ref{le:xi_t_bound}. Substituting \eqref{eqn:XiNorm} in~\eqref{eqn:PsiRecurrence} completes the proof of Part~($i$) of Lemma~\ref{le:ErrorReccursion}. Next, we prove Part~($ii$) of the lemma by defining $\widehat{\bv}_{t-1}=\bv_{t-1}/\|\bv_{t-1}\|_2$ and noting that
	\begin{align}
	\frac{\|\bxi_t\|_2^2}{\|\bv_{t-1}\|_2^2}&=  \frac{\|(1/N)\sum_{i=1}^{N}\bxi_{i,t}\|_2^2}{\|\bv_{t-1}\|_2^2} =  \frac{(1/N^2)\|\sum_{i=1}^{N}\bxi_{i,t}\|_2^2}{\|\bv_{t-1}\|_2^2}\nonumber\\
	&\stackrel{(d)}{\leq} \frac{(1/N^2)\sum_{i=1}^{N}N\|\bxi_{i,t}\|_2^2}{\|\bv_{t-1}\|_2^2} = \frac{\sum_{i=1}^{N} (\bx_{i,t}^{\tT}\bv_{t-1})^2 \|\bx_{i,t} -(\bx_{i,t}^{\tT}\widehat{\bv}_{t-1})\widehat{\bv}_{t-1}\|_2^2}{N\|\bv_{t-1}\|_2^2}\nonumber\\
	&\leq \frac{1}{N}\sum_{i=1}^{N}\|\bx_{i,t}\|_2^2 \|\bx_{i,t} -(\bx_{i,t}^{\tT}\widehat{\bv}_{t-1})\widehat{\bv}_{t-1}\|_2^2 = \frac{1}{N}\sum_{i=1}^{N}\|\bx_{i,t}\|_2^2 (\|\bx_{i,t}\|_2^2 - (\bx_{i,t}^{\tT}\widehat{\bv}_{t-1})^2)\nonumber\\
	&\leq \sum_{i=1}^{N}\frac{\|\bx_{i,t}\|_2^4}{N}\leq \max_i{\|\bx_{i,t}\|_2^4}\leq r^4.
	\end{align}
	Here, ($d$) is by using Cauchy--Schwartz inquality and the last inequality is due to Assumption~{\bf [A1]}. Now substituting this in~\eqref{eqn:PsiRecurrence} completes the proof.
\end{proof}
	
\subsection{Proof of Lemma~\ref{le:Variance_zt}}\label{secApp:App_InitialPhase_2}
We begin by writing
\begin{align*}
		\E\{(z_t - \E\{z_t\})^2|\cF_{t-1}\} &= \E\Bigg\{\Bigg(\frac{2\gamma_t (\bv_{t-1}^{\tT}\bq^*)(\bxi_t^{\tT}\bq^*)}{\|\bv_{t-1}\|_2^2}-\E\Big\{\frac{2\gamma_t (\bv_{t-1}^{\tT}\bq^*)(\bxi_t^{\tT}\bq^*)}{\|\bv_{t-1}\|_2^2}\Big\}\Bigg)^2\Bigg|\cF_{t-1}\Bigg\}\\
		&= \frac{4\gamma_t^2(\bv_{t-1}^{\tT}\bq^*)^2}{\|\bv_{t-1}\|_2^4}\E\Bigg\{\Bigg(\bxi_t^{\tT}\bq^*-\E\Big\{\bxi_t^{\tT}\bq^*\Big\}\Bigg)^2\Bigg\}.
		\end{align*}
Substituting value of $\bxi_t$ in this, we get
		\begin{align*}
		\E\{(z_t - \E\{z_t\})^2|\cF_{t-1}\} &= \frac{4\gamma_t^2(\bv_{t-1}^{\tT}\bq^*)^2}{\|\bv_{t-1}\|_2^4}\E\Bigg\{\Bigg(\Big(\frac{1}{N}\sum_{i=1}^{N}\bA_{i,t} \bv_{t-1} - \frac{\bv_{t-1}^{\tT}\frac{1}{N}\sum_{i=1}^{N}\bA_{i,t}\bv_{t-1} \bv_{t-1} }{\|\bv_{t-1}\|_2^2}\Big)^{\tT}\bq^*\\
		&\qquad\qquad-\E\Big\{\Big(\frac{1}{N}\sum_{i=1}^{N}\bA_{i,t} \bv_{t-1} - \frac{\bv_{t-1}^{\tT}\frac{1}{N}\sum_{i=1}^{N}\bA_{i,t}\bv_{t-1} \bv_{t-1} }{\|\bv_{t-1}\|_2^2}\Big)^{\tT}\bq^*\Big\}\Bigg)^2\Bigg\}\\
		&= \frac{4\gamma_t^2(\bv_{t-1}^{\tT}\bq^*)^2}{\|\bv_{t-1}\|_2^4}\E\Bigg\{\Bigg(\Big(\frac{1}{N}\sum_{i=1}^{N}\bA_{i,t} \bv_{t-1} - \frac{\bv_{t-1}^{\tT}\frac{1}{N}\sum_{i=1}^{N}\bA_{i,t}\bv_{t-1} \bv_{t-1} }{\|\bv_{t-1}\|_2^2}\Big)^{\tT}q_1\\
		&\qquad\qquad-\bv_{t-1}^{\tT}\E\Big\{\frac{1}{N}\sum_{i=1}^{N}\bA_{i,t}\Big\} \bq^* + \frac{\bv_{t-1}^{\tT}\bv_{t-1}^{\tT}\E\{\frac{1}{N}\sum_{i=1}^{N}\bA_{i,t}\}\bv_{t-1} }{\|\bv_{t-1}\|_2^2}  \bq^*\Bigg)^2\Bigg\}.
		\end{align*}
Since $\E\Big\{\frac{1}{N}\sum_{i=1}^{N}\bA_{i,t}\Big\}$ is the covariance matrix $\bSigma$, we get
		\begin{align}\label{eqn:z_var}
		\E\{(z_t - &\E\{z_t\})^2|\cF_{t-1}\}\nonumber\\
		&= \frac{4\gamma_t^2(\bv_{t-1}^{\tT}\bq^*)^2}{\|\bv_{t-1}\|_2^4}\E\Bigg\{\Bigg(\Big((\frac{1}{N}\sum_{i=1}^{N}\bA_{i,t}-\bSigma) \bv_{t-1} - \frac{\bv_{t-1}^{\tT}(\frac{1}{N}\sum_{i=1}^{N}\bA_{i,t}-\bSigma)\bv_{t-1} \bv_{t-1} }{\|\bv_{t-1}\|_2^2}\Big)^{\tT}\bq^*\Bigg)^2\Bigg\},\nonumber\\
		&= \frac{4\gamma_t^2(\bv_{t-1}^{\tT}\bq^*)^2}{\|\bv_{t-1}\|_2^4}\E\Bigg\{\Bigg(\Big({\bq^*}^{\tT}(\frac{1}{N}\sum_{i=1}^{N}\bA_{i,t}-\bSigma) \bv_{t-1} - \frac{\Big(\bv_{t-1}^{\tT}(\frac{1}{N}\sum_{i=1}^{N}\bA_{i,t}-\bSigma)\bv_{t-1}\Big) {\bq^*}^{\tT}\bv_{t-1}}{\|\bv_{t-1}\|_2^2}\Big)\Bigg)^2\Bigg\}\nonumber\\
		&\leq \frac{8\gamma_t^2(\bv_{t-1}^{\tT}\bq^*)^2}{\|\bv_{t-1}\|_2^4}\E\Bigg\{\Bigg({\bq^*}^{\tT}(\frac{1}{N}\sum_{i=1}^{N}\bA_{i,t}-\bSigma) \bv_{t-1}\Bigg)^2 + \Bigg(\frac{\Big(\bv_{t-1}^{\tT}(\frac{1}{N}\sum_{i=1}^{N}\bA_{i,t}-\bSigma)\bv_{t-1}\Big) {\bq^*}^{\tT}\bv_{t-1}}{\|\bv_{t-1}\|_2^2}\Bigg)^2\Bigg\}\nonumber\\
		&= \frac{8\gamma_t^2(\bv_{t-1}^{\tT}\bq^*)^2}{\|\bv_{t-1}\|_2^2}\E\Bigg\{\Bigg(\frac{{\bq^*}^{\tT}(\frac{1}{N}\sum_{i=1}^{N}\bA_{i,t}-\bSigma) \bv_{t-1}}{\|\bv_{t-1}\|_2}\Bigg)^2 + \Bigg(\frac{\bv_{t-1}^{\tT}(\frac{1}{N}\sum_{i=1}^{N}\bA_{i,t}-\bSigma)\bv_{t-1}}{\|\bv_{t-1}\|_2^2}\Bigg)^2 \Bigg(\frac{{\bq^*}^{\tT}\bv_{t-1}}{\|\bv_{t-1}\|_2}\Bigg)^2\Bigg\}\nonumber\\
		&\leq 8\gamma_t^2\E\Bigg\{\Bigg(\frac{{\bq^*}^{\tT}(\frac{1}{N}\sum_{i=1}^{N}\bA_{i,t}-\bSigma) \bv_{t-1}}{\|\bv_{t-1}\|_2}\Bigg)^2 + \Bigg(\frac{\bv_{t-1}^{\tT}(\frac{1}{N}\sum_{i=1}^{N}\bA_{i,t}-\bSigma)\bv_{t-1}}{\|\bv_{t-1}\|_2^2}\Bigg)^2 \Bigg\},
		\end{align}
where the last inequality in \eqref{eqn:z_var} is due to the fact that $\Bigg(\frac{{\bq^*}^{\tT}\bv_{t-1}}{\|\bv_{t-1}\|_2}\Bigg)^2\leq 1$. We can see that both the remaining terms in \eqref{eqn:z_var} are Rayleigh quotients of matrix $(\bSigma-\frac{1}{N}\sum_{i=1}^{N}\bA_{i,t})$ and hence the largest eigenvalue of $(\bSigma-\frac{1}{N}\sum_{i=1}^{N}\bA_{i,t})$ maximizes both the terms. Using this fact we get
		\begin{align*}
		\E\{(z_t - \E\{z_t\})^2|\cF_{t-1}\} \leq 16\gamma_t^2 \E\{\|\bSigma - \frac{1}{N}\sum_{i=1}^{N}\bA_{i,t}\|_2^2\}\leq 16\gamma_t^2 \E\{\|\bSigma - \frac{1}{N}\sum_{i=1}^{N}\bA_{i,t}\|_F^2\}.
		\end{align*}
Using Definition \ref{def:SampleVarianceDistributed}, we get $\E\{(z_t - \E\{z_t\})^2|\cF_{t-1}\} \leq 16\gamma_t^2 \sigma_N^2,$ which completes the proof. \qed

\subsection{Proof of Lemma~\ref{le:MGF_Bound}}\label{secApp:App_InitialPhase_3}
Using Lemma~\ref{le:ErrorReccursion}, we can write the moment generating function of $\Psi_t$ as follows:
		\begin{align}\label{eqn:Proof_MGF}
		\E\{\exp(s\Psi_t)|\cF_{t-1}\} &\leq \E\Big\{\exp\Big(s\Psi_{t-1} +s\gamma_t^2 r^4 - s z_t\Big)\Big|\cF_{t-1}\Big\} = \exp(s\Psi_{t-1}+s\gamma_t^2  r^4) \E\Big\{\exp\Big(- s z_t\Big)\Big|\cF_{t-1}\Big\}\nonumber\\
		&= \exp(s\Psi_{t-1} + s\gamma_t^2 r^4 - s\E\{z_t|\cF_{t-1}\}) \E\Big\{\exp\Big(- s (z_t - \E\{z_t\})\Big)\Big|\cF_{t-1}\Big\}.
		\end{align}
We can bound this using Bennett's inequality (Proposition \ref{prop:BennetsInequality} in Appendix~\ref{sec:App_OtherResults}), which requires the variance and range of the random variable $z_t$. We have already computed the variance of $z_t$ in Lemma~\ref{le:Variance_zt}. Next we compute the boundedness of $(z_t - \E\{z_t\})$ as follows:
		\begin{align}\label{eqn:h_def}
		\Big|z_t - \E\{z_t\}\Big|&\leq 2|z_t|\leq 2\gamma_t \|\bx_{i,t}\|_2^2\leq 2\gamma_t r^2=:h.
		\end{align}
		Here, the last inequality is due to Assumption \ref{Assum:A1}. Using parameters $\sigma_N^2$ and $h$ with Bennett's inequality, we get
		\begin{align}\label{eqn:ApplyingBennett}
		\mathbb{E}\{\exp(s\Psi_t)|\cF_{t-1}\}\leq \exp\Bigg(s\Psi_{t-1} - s\E\{z_t|\cF_{t-1}\} + s\gamma_t^2 r^4 +s^2\gamma_t^2 \sigma_N^2\Bigg(\frac{e^{s h} - 1 - sh}{(sh)^2}\Bigg)\Bigg).
		\end{align}
		For $L\geq L_1+L_2$, where $L_1$ and $L_2$ are given by \eqref{eqn:LowerboundL}, we show in Proposition~\ref{prop:MGF_t0} in Appendix~\ref{sec:App_OtherResults} that $(\frac{e^{sh}-1-sh}{(sh)^2})\leq 1$ for $s \in \bbS$. This implies
		\begin{align*}
		\mathbb{E}\{\exp(s\Psi_t)|\cF_{t-1}\}&\leq \exp\Bigg(s\Psi_{t-1} - s\E\{z_t|\cF_{t-1}\} + s\gamma_t^2 r^4 +s^2\gamma_t^2 \sigma_N^2\Bigg),
		\end{align*}
which completes the proof of the lemma. \qed	
	
	\section{Proofs of Lemmas for the Intermediate Epoch}\label{sec:App_Intermediate}
	\subsection{Proof of Lemma~\ref{le:le4}}\label{sec:App_Intermediate_1}
		Using Lemma \ref{le:MGF_Bound}, we have
		\begin{align}
		\E\{e^{s\Psi_{t}}\big|\cF_{t-1}\}&\leq \exp\Bigg(s\Bigg(\Psi_{t-1} + \gamma_t^2 r^4 - \E\{z_t|\cF_{t-1}\} + s \gamma_t^2 \sigma_N^2 \Bigg) \Bigg)\nonumber\\
		&\stackrel{(a)}\leq \exp\Bigg(s\Bigg(\Psi_{t-1} - 2\gamma_t \Big(\lambda_1 - \lambda_2\Big)\Psi_{t-1}\Big(1 - \Psi_{t-1}\Big) + \gamma_t^2 r^4 + s \gamma_t^2 \sigma_N^2 \Bigg)\Bigg)\nonumber\\
		&\stackrel{(b)}\leq \exp\Bigg(s\Bigg(\Psi_{t-1} - \frac{c_0 \Psi_{t-1}\Big(1 - \Psi_{t-1}\Big)}{t+L} + \frac{c^2 r^4}{(t+L)^2} + \frac{s c^2 \sigma_N^2}{(t+L)^2} \Bigg)\Bigg).
		\end{align}
		Here, $(a)$ is due to \cite[Lemma A.3]{Balsubramani2015} and $(b)$ is by substituting $\gamma_t=c/(t+L)=c_0/2(\lambda_1 - \lambda_2)(t+L)$. Finally, for $\omega\in \Omega_t^{'}$ we have $\Psi_{t-1}(\omega)\leq 1-\epsilon_j$. Now taking expectation over $\Omega_t^{'}$, we get the desired result. \qed

	\subsection{Proof of Lemma~\ref{le:MGF_n}}\label{sec:App_Intermediate_2}
		Define $\alpha_t := 1 - \frac{c_0 \epsilon_j}{t+L}$ and $\zeta_t(s) :=\frac{s c^2 r^4}{(t+L)^2}+\frac{s^2 c^2 \sigma_N^2}{(t+L)^2}$. Substituting $\alpha_t$ and $\zeta_t(s)$ in Lemma~\ref{le:le4}, we get
		\begin{align}\label{eqn:Psi_MGF}
		\E_t\big\{e^{s\Psi_t}\big\}\leq \E_{t}\big\{e^{s\alpha_t \Psi_{t-1}}\big\}\exp{\big(\zeta_t(s)\big)}\leq \E_{t-1}\big\{e^{s\alpha_t \Psi_{t-1}}\big\}\exp{\big(\zeta_t(s)\big)}.
		\end{align}
		Note that the second inequality in \eqref{eqn:Psi_MGF} is due to \cite[Lemma 2.8]{Balsubramani2015}. Applying this procedure repeatedly yields
		\begin{align*}
		\E_t\big\{e^{s\Psi_t}\big\}&\leq \E_{t_{j}+1}\big\{\exp{\big(s\Psi_{t_j} \alpha_t\dots\alpha_{t_{j}+1}\big)}\big\}\exp{\big(\zeta_t(s)\big)}\dots\exp{\big(\zeta_{t_{j}+1}\big(s\alpha_t\dots\alpha_{t_{j}+1}\big)\big)}\\
		&\leq \E_{t_{j}+1}\big\{\exp{\big(s\Psi_{t_j} \alpha_t\dots\alpha_{t_j+1}\big)}\big\}\exp{\big(\zeta_t(s)\big)}\dots\exp{\big(\zeta_{t_j+1}\big(s\big)\big)}.
		\end{align*}
		Substituting values of $\alpha_t$ and $\zeta_t(s)$ in the above, we get
		\begin{align}\label{eqn:l5e1}
		\E_t\big\{e^{s\Psi_t}\big\}&\leq \E_{t_{j}+1}\Big\{\exp{\Big(s\Psi_{t_j} \Big(1-\frac{c_0 \epsilon_j}{t+L}\Big)\dots \Big(1-\frac{c_0 \epsilon_j}{t_j + L + 1}\Big)\Big)}\Big\}\nonumber\\
		&\qquad\qquad\exp{\Bigg(\Big(s c^2 r^4+s^2 c^2 \sigma_N^2\Big) \Big(\frac{1}{(t+L)^2}+\dots+\frac{1}{(t_j + L + 1)^2}\Big) \Bigg)}\nonumber\\
		&\leq \exp{\Big(s(1-\epsilon_j)\exp{\Big(-c_0 \epsilon_j\Big(\frac{1}{t+L}+\dots+\frac{1}{t_j +L + 1}\Big)\Big)}\Big)}\nonumber\\
		&\qquad\qquad\exp{\Bigg(\Big(s c^2 r^4+s^2 c^2 \sigma_N^2\Big) \Big(\frac{1}{(t + L)^2}+\dots+\frac{1}{(t_j + L + 1)^2}\Big) \Bigg)}.
		\end{align}
Here, the last inequality is true because $\Psi_{t_j}(\omega)\leq 1-\epsilon_j$ for $\omega\in\Omega^{'}_{t_j + 1}$ and $1-x\leq e^{-x}$ for $x\leq 1$. Next we bound the summations in \eqref{eqn:l5e1} as follows:
		$$\frac{1}{t+L}+\dots+\frac{1}{t_j + L + 1}\geq \int_{t_j + 1}^{t + 1}{\frac{dx}{x+L}}=\ln{\frac{t + L + 1}{t_j + L + 1}},$$
		$$\frac{1}{(t+L)^2}+\dots+\frac{1}{(t_j + L + 1)^2}\leq \int_{t_j}^{t}{\frac{dx}{(x+L)^2}}=\frac{1}{t_j+L}-\frac{1}{t+L}.$$
		Substituting these bounds in \eqref{eqn:l5e1}, we get the desired result. \qed

	\subsection{Proof of Lemma~\ref{le:MGF_nj1}}\label{sec:App_Intermediate_3}
			This lemma uses Lemma~\ref{le:MGF_n} and deals with a specific value of $t =t_{j+1}$.
		For $t=t_{j+1}$, \eqref{eqn:MGF_n} gives
		\begin{align}\label{eqn:MGF_nj1}
		\E_{t_j+1}\{e^{s\Psi_{t_j+1}}\}\leq \exp{\Bigg(s(1-\epsilon_j)\Bigg(\frac{t_j + L + 1}{t_{j+1} + L + 1}\Bigg)^{c_0 \epsilon_j}+ \Bigg(sc^2 r^4+s^2 c^2 \sigma_N^2\Bigg)\Bigg(\frac{1}{t_j + L}-\frac{1}{t_{j+1} + L}\Bigg)\Bigg)}.
		\end{align}
		Using conditions \ref{Condition:C1} and \ref{Condition:C2} and the fact that $e^{-2x}\leq 1 - x$ for $0\leq x \leq 3/4$, we get
		$$(1-\epsilon_j)\Big(\frac{t_j + L + 1}{t_{j+1} + L + 1}\Big)^{c_0 \epsilon_j}\leq e^{-\epsilon_j}(e^{-5/c_0 })^{c_0 \epsilon_j}=e^{-6\epsilon_j}\leq 1 - 3\epsilon_j \leq 1 - \epsilon_{j+1}-\epsilon_j.$$
		Substituting this in \eqref{eqn:MGF_nj1}, we obtain the desired result. \qed

\subsection{Proof of Lemma~\ref{le:Probability_tj}}\label{sec:App_Intermediate_4}
		Constructing a supermartingale sequence $M_t$ in the same way as we did in Theorem~\ref{thm:ProbabilityBound} for $s \in \bbS$ and applying Doob's martingale inequality, we get
		\begin{align*}
		\bPr_{t_j}\Big(\sup_{t\geq t_j} \Psi_t \geq 1 - \epsilon_j\Big)&\leq \bPr_{t_j}\Big(\sup_{t\geq t_j} M_t \geq e^{s(1-\epsilon_j)}\Big) \leq \frac{\E\{M_{t_j}\}}{e^{s(1-\epsilon_j)}}\\
		&=\frac{\E\big\{\exp{(s\Psi_{t_j}+s\tau_{t_j})}\big\}}{e^{s(1-\epsilon_j)}} =\frac{\E\big\{\exp{(s\Psi_{t_j})}\big\}\exp{(s\tau_{t_j})}}{e^{s(1-\epsilon_j)}}.
		\end{align*}
		Using Lemma \ref{le:MGF_nj1} then results in
		$$\bPr_{t_j}\Big(\sup_{t\geq t_j} \Psi_t \geq 1 - \epsilon_j\Big)\leq \frac{1}{e^{s(1-\epsilon_j)}}\exp{\Bigg(s(1-\epsilon_j)-s\epsilon_{j-1}+\Big(s c^2 r^4+s^2 c^2 \sigma_N^2\Big)\Big(\frac{1}{t_{j-1}+L} - \frac{1}{t_j+L}\Big)+s\tau_{t_j}\Bigg)}.$$
        Substituting a bound on $\tau_{t_j}$ from Theorem \ref{thm:ProbabilityBound} (see, e.g., the discussion around~\eqref{eqn:tauUpperBound}), we get
		\begin{align*}
		\bPr_{t_j}\Big(\sup_{t\geq t_j} \Psi_t \geq 1 - \epsilon_j\Big)&\leq \exp{\Bigg(-s\epsilon_{j-1}+\Big(s c^2  r^4+s^2 c^2 \sigma_N^2\Big)\Big(\frac{1}{t_{j-1}+L} - \frac{1}{t_j+L}\Big)+s\Big( c^2 r^4 +s c^2 \sigma_N^2\Big)\frac{1}{t_j+L}\Bigg)}\\
		&= \exp{\Bigg(-s\epsilon_{j-1}+s\Big(c^2 r^4 + s c^2 \sigma_N^2\Big)\frac{1}{t_{j-1}+L}\Bigg)}.
		\end{align*}
		Substituting $s=(2/\epsilon_0)\ln{(4/\delta)}$ and using the lower bound on $L$, we get (see Proposition~\ref{prop:Probability_tj_1} in Appendix~\ref{sec:App_OtherResults} for formal verification)
		\begin{align*}
		\bPr_{t_j}\Big(\sup_{t\geq t_j} \Psi_t \geq 1 - \epsilon_j\Big)\leq \exp{\Big(-\frac{s\epsilon_{j-1}}{2}\Big)}=\Bigg(\frac{\delta}{4}\Bigg)^{\epsilon_{j-1}/\epsilon_0}\leq \frac{\delta}{2^{j+1}}.
		\end{align*}
		Summing over $j$ completes the proof of the lemma. \qed

\section{Proofs for the Final Epoch}\label{sec:App_FinalEpoch}

\begin{proof}[Proof of Lemma~\ref{le:FinalRate}]
From Lemma \ref{le:ErrorReccursion}, Part ($i$), we have
		\begin{align*}
		\Psi_t \leq \Psi_{t-1} + 4\gamma_t^2 \Big(\Big\|\frac{1}{N}\sum_{i=1}^{N}{\bA_{i,t}} - \bSigma\Big\|_F^2 + \lambda_1^2\Psi_{t-1}\Big) - z_{t}.
		\end{align*}
		Taking expectation conditioned on $\cF_{t-1}$, we get
		\begin{align*}
		\E\{\Psi_t|\cF_{t-1}\} \leq \Psi_{t-1}(1 + \gamma_t^2 \lambda_1^2) + 4\gamma_t^2 \sigma_N^2 - \E\big\{z_t\big|\cF_{t-1}\big\},
		\end{align*}
		where the second term is due to Lemma~\ref{le:xi_t_bound}. Now using upper bound on $-\E\big\{z_t\big|\cF_{t-1}\big\}$ from \cite[Lemma A.4]{Balsubramani2015}, we get the following:
		\begin{align*}
		\E\{\Psi_t|\cF_{t-1}\} &\leq \Psi_{t-1}(1 + \gamma_t^2 \lambda_1^2) + 4\gamma_t^2 \sigma_N^2 - 2\gamma_t (\lambda_1 - \lambda_2) \Psi_{t-1}(1-\Psi_{t-1})\\
		&= \Psi_{t-1}\Big(1 + \gamma_t^2 \lambda_1^2 - 2\gamma_t (\lambda_1 - \lambda_2)(1 - \Psi_{t-1})\Big) + 4\gamma_t^2 \sigma_N^2.
		\end{align*}
Finally, taking expectation over $\Omega_t^{'}$, substituting $\gamma_t=c_0/(2(t+L)(\lambda_1 - \lambda_2))$, and using the facts that $\Omega_t^{'}$ is $\cF_{t-1}$-measurable and for $t>t_J$, $\Psi_{t-1}\leq 1/2$ and we lie in sample space $\Omega_t^{'}$ with probability greater than $1-\delta$ (Theorem~\ref{thm:ProbabilityBound}), we obtain
		\begin{align*}
		\E_t\{\Psi_t\}&\leq \E_{t}\Bigg\{\Psi_{t-1}\Big(1+\frac{c_0^2 \lambda_1^2}{2(t+L)^2 (\lambda_1 - \lambda_2)^2}-\frac{c_0}{2 (t+L)}\Big)\Bigg\} + \frac{4c^2 \sigma_N^2}{(t+L)^2}\\
		&= \Bigg(1+\frac{c_0^2 \lambda_1^2}{2(t+L)^2 (\lambda_1 - \lambda_2)^2}-\frac{c_0}{2 (t+L)}\Bigg)\E_{t}\{\Psi_{t-1}\}+ \frac{4c^2 \sigma_N^2}{(t+L)^2} \\
		&\leq \Bigg(1+\frac{c_0^2 \lambda_1^2}{2(t+L)^2 (\lambda_1 - \lambda_2)^2}-\frac{c_0}{2 (t+L)}\Bigg)\E_{t-1}\{\Psi_{t-1}\}+ \frac{4c^2 \sigma_N^2}{(t+L)^2}.
		\end{align*}
This completes the proof of the lemma.
\end{proof}

\begin{proposition}\label{prop:recursion}
Let $a_1, b>0$ and $a_2>1$ be some constants. Consider a nonnegative sequence $(u_t:t>t_J)$ that satisfies
$$u_t\leq \Big(1+\frac{a_1}{(t+L)^2}-\frac{a_2}{t+L}\Big)u_{t-1}+\frac{b}{(t+L)^2}.$$
Then we have:
$$u_t\leq \Bigg(\frac{L+1}{t + L + 1}\Bigg)^{a_2}\exp \Big(\frac{a_1}{L + 1}\Big)u_{0}
+ \frac{1}{(t+L+1)}\exp\big(\frac{a_1}{L + 1}\big)\Big(\frac{L+2}{L+1}\Big)^2\frac{b}{a_2-1}.$$
\end{proposition}
\begin{proof}
	Recursive application of the bound on $u_t$ gives:
	\begin{align}\label{eqn:recur_1}
	u_t\leq \Bigg(\prod_{i=t_J+1}^{t}\Big(1 + \frac{a_1}{(i + L)^2} - \frac{a_2}{i + L}\Big)\Bigg)u_{t_0}+\sum_{i=t_J + 1}^{t}\frac{b}{(i+L)^2}\Bigg(\prod_{j=i+1}^{t}\Big(1 + \frac{a_1}{(j+L)^2} - \frac{a_2}{j+L}\Big)\Bigg).
	\end{align}
	Using \cite[Lemma D.1]{Balsubramani2015} we can bound the product terms as
	\begin{align}\label{eqn:recursion}
	\prod_{j=i+1}^{t}\Big(1 + \frac{a_1}{(j+L)^2} - \frac{a_2}{j+L}\Big)&\leq \exp\Bigg(\sum_{j=i}^{t}\frac{a_1}{(j+L)^2}-\sum_{j=i}^{t}\frac{a_2}{j+L}\Bigg)\nonumber\\
	&\leq\Bigg(\frac{i + L + 1}{t + L + 1}\Bigg)^{a_2}\exp\Bigg(\sum_{j=i}^{t}\frac{a_1}{(j+L)^2}\Bigg).
	\end{align}
	Next, we bound the last term here as
	\begin{align*}
		\exp\Bigg(\sum_{j=i}^{t}\frac{a_1}{(j+L)^2}\Bigg)\leq \exp\Bigg(\int_{i+1}^{t+1}\frac{a_1}{(x+L)^2}dx\Bigg)=\exp \Big(\frac{a_1}{i + L+1}-\frac{a_1}{t + L +1}\Big)\leq \exp\Big(\frac{a_1}{i + L + 1}\Big).
	\end{align*}
	Substituting this in~\eqref{eqn:recursion} we get
	\begin{align*}
	\prod_{j=i+1}^{t}\Big(1 + \frac{a_1}{(j + L)^2} - \frac{a_2}{j+L}\Big)\leq \Bigg(\frac{i + L + 1}{t + L + 1}\Bigg)^{a_2}\exp \Big(\frac{a_1}{i + L+1}\Big).
	\end{align*}
	Substituting this in~\eqref{eqn:recur_1} we get
	\begin{align*}
	u_t&\leq \Bigg(\frac{t_J + L + 1}{t + L + 1}\Bigg)^{a_2}\exp \Big(\frac{a_1}{t_J + L + 1}\Big)u_{t_J}
	+ \sum_{i=t_J+1}^{t}\frac{b}{(i+L)^2}\Bigg(\prod_{j=i+1}^{t}\Big(1 + \frac{a_1}{(j+L)^2}-\frac{a_2}{j+L}\Big)\Bigg)\\
	&\leq \Bigg(\frac{t_J + L+1}{t + L + 1}\Bigg)^{a_2}\exp \Big(\frac{a_1}{t_J + L + 1}\Big)u_{t_J}
	+ \sum_{i=t_J+1}^{t}\frac{b}{(i+L)^2}\Bigg(\frac{i+L+1}{t+L+1}\Bigg)^{a_2}\exp\big(\frac{a_1}{i+L+1}\big)\\
	&\leq \Bigg(\frac{t_J + L+1}{t + L + 1}\Bigg)^{a_2}\exp \Big(\frac{a_1}{t_J + L + 1}\Big)u_{t_J}
	+ \exp\Big(\frac{a_1}{t_J + L + 1}\Big)\frac{b}{(t+L+1)^{a_2}}\sum_{i=1}^{t}\frac{(i+L+1)^{a_2}}{(i+L)^2}\\
	&\leq \Bigg(\frac{t_J + L + 1}{t + L + 1}\Bigg)^{a_2}\exp \Big(\frac{a_1}{t_J + L + 1}\Big)u_{t_J}
	+ \exp\Big(\frac{a_1}{t_J + L + 1}\Big)\frac{b}{(t+L+1)^{a_2}}\Big(\frac{L+2}{L+1}\Big)^2\sum_{i=1}^{t}(i+L+1)^{a_2-2}.
	\end{align*}
Again applying \cite[Lemma D.1]{Balsubramani2015}, we get the final result as follows
    \begin{align*}
    u_t &\leq \Bigg(\frac{t_J + L+1}{t + L + 1}\Bigg)^{a_2}\exp \Big(\frac{a_1}{t_J + L + 1}\Big)u_{t_J}
    + \exp\Big(\frac{a_1}{t_J + L + 1}\Big)\frac{b}{(t+L+1)^{a_2}}\Big(\frac{L+2}{L+1}\Big)^2\frac{(t+L+1)^{a_2-1}}{a_2-1}\\
    &= \Bigg(\frac{t_J + L+1}{t + L + 1}\Bigg)^{a_2}\exp \Big(\frac{a_1}{t_J + L + 1}\Big)u_{t_J}
    + \frac{1}{(t+L+1)}\exp\big(\frac{a_1}{t_J + L + 1}\big)\Big(\frac{L+2}{L+1}\Big)^2\frac{b}{a_2-1}.
    \end{align*}
This completes the proof of the proposition.
\end{proof}	
	
\section{Other Auxiliary Results}\label{sec:App_OtherResults}
\begin{proposition}[Bennett's Inequality \citep{boucheron2013concentration}]\label{prop:BennetsInequality}
Consider a zero-mean, bou\-nded random variable $X_i \in \R$ (i.e., $|X_i|\leq h$ almost surely) with variance $\sigma_i^2$. Then for any $s\in \R$, we have
$$\E\big\{e^{s X_i}\big\}\leq \exp{\Bigg(\sigma_i^2 s^2 \Big(\frac{e^{s h} - 1 - sh}{(s h)^2}\Big)\Bigg)}.$$
\end{proposition}

\begin{proposition}\label{prop:MGF_t0}
Let $h := 2\gamma_t r^2$ and $s \in \big\{d/4\epsilon, (2/\epsilon_0)\ln(4/\delta)\big\}$. It then follows that $\frac{e^{sh}-1-sh}{(sh)^2}\leq 1$.
\end{proposition}
\begin{proof}
It is straightforward to see that $\frac{e^{sh}-1-sh}{(sh)^2}\leq 1$ as long as $sh\leq 7/4$. Therefore, in order to prove this proposition, it suffices to show that the lower bound on $L$ implies $sh\leq 7/4$ for $s \in \big\{d/4\epsilon, (2/\epsilon_0)\ln(4/\delta)\big\}$. We establish this claim as two separate cases for the two values of $s$.	
	
\emph{\underline{Case I:}} For $s=d/4\epsilon$, substituting the value of $h$ gives us
	\begin{align*}
	sh =\frac{d \gamma_t  r^2}{2\epsilon}=\frac{d c r^2}{2(t+L)\epsilon}\leq \frac{d c r^2}{2L\epsilon}\leq \frac{d c r^2}{2\epsilon L_{1}} \leq \frac{dc r^2}{2\epsilon}\frac{\epsilon}{8 d r^4\max(1, c^2)\ln(4/\delta)}\leq \frac{1}{16 \ln (4/\delta)}\leq \frac{7}{4}.
	\end{align*}
	
\emph{\underline{Case II:}} For $s=(2/\epsilon_0)\ln(4/\delta)$, we obtain
	\begin{align*}
	sh =\frac{2 \ln(4/\delta)c r^2}{\epsilon_0 (t+L)}\leq \frac{2\ln(4/\delta)c r^2}{\epsilon_0 L_{1}} \leq \frac{2\ln(4/\delta)c r^2}{\epsilon_0}\frac{\epsilon_{0}}{8 r^4 \max(1, c^2)\ln\frac{4}{\delta}}\leq \frac{1}{4}\leq\frac{7}{4}.
	\end{align*}
This completes the proof of the proposition.
\end{proof}

\begin{proposition}\label{prop:tau_t0}
	Assuming $L\geq \frac{8 dr^4 \max(1, c^2)}{\epsilon}\ln\frac{4}{\delta}+\frac{8 d^2 \sigma_N^2 \max(1, c^2)}{\epsilon^2}\ln\frac{4}{\delta}$ and the parameter $s=d/4\epsilon$, we have $\frac{c^2}{L}\Big(r^4 + s\sigma_N^2\Big)\leq \frac{\epsilon}{d}$.
\end{proposition}
\begin{proof}
	We prove this by proving the following two statements:
	$$\frac{c^2 r^4}{L}\leq\frac{c^2 r^4}{L_{1}}\leq \frac{\epsilon}{2d}\quad \text{and}\quad \frac{sc^2 \sigma_N^2}{L}\leq\frac{sc^2 \sigma_N^2}{L_{2}}\leq \frac{\epsilon}{2d}.$$
We start by proving the first statement: $\frac{c^2 r^4}{L_{1}}\leq c^2 r^4 \frac{\epsilon}{8 dr^4 \max(1, c^2)\ln\frac{4}{\delta}}\leq \frac{\epsilon}{2d}.$ Next, we prove the second statement as follows: $\frac{c^2 s \sigma_N^2}{L_{2}}\leq \frac{c^2 d \sigma_N^2}{4\epsilon}\frac{\epsilon^2}{8 d^2 \sigma_N^2 \max(1, c^2)\ln\frac{4}{\delta}}\leq \frac{\epsilon}{2d}.$ This completes the proof.
\end{proof}

\begin{proposition}\label{prop:Probability_tj_1}
	For $L\geq \frac{8 r^4 \max(1, c^2)}{\epsilon_0}\ln\frac{4}{\delta}+\frac{8 \sigma_N^2 \max(1, c^2)}{\epsilon_0^2}\ln\frac{4}{\delta},$ we have
	\begin{enumerate}
		\item[(i)] $\frac{c^2 r^4}{(t_{j-1}+L)}\leq \frac{\epsilon_0}{4}$, and
		\item[(ii)] $\frac{2 c^2 \sigma_N^2}{\epsilon_0 (t_{j-1}+L)}\ln{\frac{4}{\delta}}\leq \frac{\epsilon_0}{4}$.
	\end{enumerate}
\end{proposition}
\begin{proof}
We begin by noting that
\begin{align*}
\frac{c^2 r^4}{(t_{j-1}+L)} \leq \frac{2c^2 r^4}{L}\leq \frac{2c^2 r^4}{ L_{1}} \leq {2c^2 r^4}\frac{\epsilon_0}{8 r^4 \max(1, c^2)\ln\frac{4}{\delta}}\leq \frac{\epsilon_{0}}{4}.
\end{align*}
Next we prove the second statement as follows:
\begin{align*}
\frac{2 c^2 \sigma_N^2}{\epsilon_0 (t_{j-1}+L)}\ln{\frac{4}{\delta}}\leq \frac{2 c^2 \sigma_B^2}{\epsilon_0 L}\ln{\frac{4}{\delta}}\leq \frac{2 c^2 \sigma_N^2}{\epsilon_0 L_{2}}\ln{\frac{4}{\delta}} \leq \frac{2 c^2 \sigma_N^2}{\epsilon_0}\ln{\frac{4}{\delta}}\frac{\epsilon_0^2}{8\sigma_N^2\max(1,c^2)\ln(4/\delta)}\leq  \frac{\epsilon_0}{4}.
\end{align*}
This completes the proof of the proposition.
\end{proof}
\end{appendices}

\end{document}